\documentclass[final,12pt]{alt2021}

\usepackage{enumitem}
\usepackage[utf8]{inputenc} 
\usepackage[T1]{fontenc}    
\usepackage{hyperref}       
\usepackage{url}            
\usepackage{booktabs}       
\usepackage{amsfonts}       
\usepackage{amssymb}
\usepackage{nicefrac}       
\usepackage{microtype}      
\usepackage{graphicx}
\usepackage{times}

\definecolor{dmb}{rgb}{0.0, 0.2, 0.4}

\usepackage{amsmath}
\usepackage{algorithm,algpseudocode}
\usepackage{wrapfig}    

\usepackage{centernot} 
\newcommand{\notiff}{%
  \mathrel{{\ooalign{\hidewidth$\not\phantom{"}$\hidewidth\cr$\iff$}}}}

\newtheorem*{observation}{Observation}

\providecommand{\doarxiv}{true}
\usepackage{ifthen}
\newboolean{isarxiv}
\setboolean{isarxiv}{\doarxiv} 
\ifthenelse{\boolean{isarxiv}}{%
\newcommand{\arxiv}[1]{#1}
\newcommand{\notarxiv}[1]{}
}{
\newcommand{\arxiv}[1]{}
\newcommand{\notarxiv}[1]{#1}
}
\newcommand{\arxivalt}[2]{\ifthenelse{\boolean{isarxiv}}{#1}{#2}}
\newcommand{\arxivaltr}[2]{\ifthenelse{\boolean{isarxiv}}{#2}{#1}}

\title{Submodular Combinatorial Information Measures with Applications in Machine Learning}

\altauthor{
\Name{Rishabh Iyer}
\Email{rishabh.iyer@utdallas.edu} \\
\addr Department of Computer Science,
University of Texas at Dallas
\AND
\Name{Ninad Khargonkar}
\Email{Ninadarun.khargonkar@utdallas.edu}\\
\addr Department of Computer Science,
University of Texas at Dallas
\AND
\Name{Jeff Bilmes} 
\Email{bilmes@uw.edu} \\
\addr Department of ECE,
University of Washington, Seattle
\AND
\Name{Himanshu Asnani}
\Email{himanshu.asnani@tifr.res.in} \\
\addr School of Technology and Computer Science, Tata Institute of Fundamental Research, Mumbai
}

\begin{document}

\maketitle

\begin{abstract}
Information-theoretic quantities like entropy and mutual information have found numerous uses in machine learning. It is well known that there is a strong connection between these entropic quantities and submodularity since entropy over a set of random variables is submodular. In this paper, we study combinatorial information measures defined over sets of (not necessarily random) variables. These measures strictly generalize the corresponding entropic measures since they are all parameterized via submodular functions that themselves strictly generalize entropy. Critically, we show that, unlike entropic mutual information in general, the submodular mutual information is actually submodular in one argument, holding the other fixed, for a large class of submodular functions whose third-order partial derivatives satisfy a non-negativity property (also called second-order supermodular functions).  We study specific instantiations of the submodular information measures, and see that they all have mathematically intuitive and practically useful expressions. Regarding applications, we connect the maximization of submodular (conditional) mutual information to problems such as mutual-information-based, query-based, and privacy preserving summarization --- and we connect optimizing the multi-set submodular mutual information to clustering and robust partitioning.\looseness-1

\looseness-1
\end{abstract}

\section{Introduction}\label{sec:introduction}

Submodular functions generalize a number of combinatorial and information theoretic functions such as entropy, set cover, facility location, graph cut, and provide a general class of expressive models. They model aspects like diversity, coverage, information \citep{linClassSubmodularFunctions2011, tschiatschekLearningMixturesSubmodular2014a}, attractive potentials \citep{iyerOptimalAlgorithmsConstrained2014} and cooperation \citep{jegelkaSubmodularitySubmodularEnergies2011}. They are closely related to convexity and concavity~\citep{lovasz1983submodular,iyer2015polyhedral} and enable efficient optimization algorithms with guarantees both in the minimization \citep{fujishige2005submodular} and maximization settings \citep{krauseSubmodularFunctionMaximization2013, leeNonmonotoneSubmodularMaximization2009, buchbinderTightLinearTime2015}. Submodular Functions also occur naturally in several machine learning applications such as sensor placement \citep{krauseRobustSubmodularObservation, krauseNearoptimalNonmyopicValue, krauseNearOptimalSensorPlacementsa}, structured learning of graphical models \citep{narasimhanPAClearningBoundedTreewidth2004}, social networks \citep{kempeMaximizingSpreadInfluence2003}, summarization \citep{linClassSubmodularFunctions2011, tschiatschekLearningMixturesSubmodular2014a, xuGazeenabledEgocentricVideo2015}, and data subset selection~\citep{wei2015submodularity}, to name only a few.

Denote $\Omega$ as a \emph{ground-set} of $n$ data points $\Omega = \{x_1, x_2, x_3,...,x_n \}$ and a set function $f:
 2^{\Omega} \xrightarrow{} \mathbb{R}$. We say that $f$ is \emph{normalized} if $f(\emptyset) = 0$ and $f$ is \textbf{subadditive} if $f(S) + f(T) \geq f(S \cup T)$, holds for all $S,T \subseteq \Omega$. Define the \textbf{first-order partial derivative} (equivalently, gain) of an element $j \notin S$ in the context $S$ as $f(j | S) = f(S \cup j) - f(S)$. The function $f$ is \textbf{submodular} \citep{fujishige2005submodular} if for all $S, T \subseteq \Omega$, it holds that $f(S) + f(T) \geq f(S \cup T) + f(S \cap T)$. An identical characterization of submodularity is the \textbf{diminishing marginal returns} property, namely $f(j | S) \geq f(j | T)$ for all $S \subseteq T, j \notin T$. Also, define the \textbf{second-order partial derivatives} in the context of $S$ as  
$f^{(2)}(j,k;S) = f(j|S \cup k) - f(j |S)$. Another equivalent definition of submodularity is that the second-order partial derivatives in the context of any set $S$ are always non-positive: $f^{(2)}(j,k;S) \leq 0$. Next, $f$ is \textbf{supermodular} if $-f$ is submodular, and $f$ is said to be \textbf{monotone} if $f(X) \leq f(Y)$, $\forall X\subseteq Y\subseteq \Omega$ (equivalently, $f(j|S) \geq 0$ for all $j \notin S$ and $S \subseteq \Omega$). Throughout the paper we assume that $f$ is a \emph{normalized}, \emph{monotone} and \emph{submodular} function unless specified otherwise.\looseness-1

Information theoretic concepts \citep{1994thomascoverInformationTheory,yeung2008information, polyanskiy2014lecture} such as entropy, mutual information, conditional mutual information and independence have also been extensively analyzed and used in machine learning settings. The connection between submodular optimization and information theory is even more evident when seen from the perspective of information measures over a set of random variables. Given a set of random variables $\{X_1, \dots, X_n\}$, we denote $X_A$ as the set of random variables indexed by the set $A \subseteq \Omega$. Then the entropy~\citep{1994thomascoverInformationTheory} function $H(X_A)$, and the mutual information between a set of variables and the complement set $I(X_A; X_{\Omega \backslash A})$, are both submodular functions~\citep{fujishige2005submodular}. These have been widely used in applications such as sensor placement~\citep{krauseNearOptimalSensorPlacementsa,krauseRobustSubmodularObservation}, feature selection~\citep{krauseNearoptimalNonmyopicValue,balaganiFeatureSelectionCriterion2010, hanchuanpengFeatureSelectionBased2005a, liuSubmodularFeatureSelection2013}, observation selection, and causal modeling \citep{zhouCausalMeetsSubmodular, steudelCausalMarkovCondition2010}. 

In this paper, we generalize concepts like conditional gain, (conditional) mutual information, total correlation, and the variation-of-information metric to general submodular functions. Many of the properties, results, and inequalities that hold for entropy, (conditional) mutual information etc.\ are closely associated with the underlying submodularity of the entropy function. Note that while we draw inspiration from the information theoretic quantities, there is however one fundamental difference. Entropic quantities such as entropy and mutual information are defined on sets of random variables and hence have well-defined statistical  interpretations. The generalizations discussed in this paper, however, assume that the submodular functions are defined on subsets of a not-necessarily stochastic set of elements $\{x_1, \dots, x_n\}$ and hence the generalizations do not necessarily have immediate stochastic implications, although they do always have combinatorial implications. We use the constructs defined in this paper in a number of optimization problems and applications related to data summarization, data selection, clustering and partitioning.

The submodular information measures we study in this work have been investigated before in special cases. \citep{steudelCausalMarkovCondition2010} generalizes an information measure to elements of any ground set, primarily  in order to introduce a causal Markov condition not over random variables. Also, in \citep{guillory2011-active-semisupervised-submodular}, an objective that corresponds to our submodular mutual information was used to show error bounds and hardness for general batch active semi-supervised learning. \citep{cunningham1983decomposition} introduce the symmetric submodular mutual information (what they call the connectivity function) and use it in studying the decomposition of submodular functions. \citep{mcgill1954multivariate, shannon1948mathematical} were the first to show the submodular inequalities of the entropy function, and \citep{zhang1997non} showed that the class of polymatroid functions (i.e., those that are normalized, monotone, non-negative, and submodular) are strictly more general than entropy (since entropy must satisfy certain inequalities not required by a polymatroid function), thereby ensuring that the combinatorial information measures studied herein are also \emph{strictly} more general. Finally, \citep{bilmes2017deep} uses the concept of the multi-set total correlation~\citep{watanabe1960information} to study properties of families of deep submodular functions. Finally, \cite{levin2020online} also study the submodular mutual information (which they call the \emph{Mutual Coverage}). They study a few of its properties and use it in characterizing an LP relaxation to the online submodular cover problem.\looseness-1 

This paper provides a complete picture of the combinatorial information measures defined via general submodular functions, by studying various properties, examples, optimization problems, and applications. A road-map of this paper is as follows.
Section~\ref{sec:def-gen} introduces the information measures. Section~\ref{sec:gen-properties} takes a closer look at their properties. Section~\ref{sec:examples} then instantiates the submodular information measures on a number of submodular functions. In Section~\ref{sec:optimization-problems}, we look at a number of optimization problems around optimizing the submodular mutual information, its conditional variant, and its multi-set extension.  We then connect them to applications in summarization and data selection. The proofs of all the results, along with some additional properties, examples, and optimization problems are in the Appendices. Appendix~\ref{sec:app-proofs-sec3} covers the proofs of the results from Section~\ref{sec:gen-properties}, along with some more properties of the submodular information measures. Appendix~\ref{sec:app-proofs-sec4} studies proofs of the instantiations of the submodular information measures presented in Section~\ref{sec:examples}, and Appendix~\ref{sec:app-proofs-sec5} presents the proofs for the results shown in Section~\ref{sec:optimization-problems} along with some more variants of optimization problems.\looseness-1

\section{Submodular Information Theoretic Quantities}\label{sec:def-gen}

We begin by introducing the combinatorial information measures parametrized by submodular functions. Since monotone non-decreasing submodular functions generalize entropy~\citep{zhang1997non}, the measures below generalize the corresponding purely entropic measures.

\paragraph{Submodular Information Functions:} Monotone non-negative non-decreasing submodular (or ``polymatroid'' \citep{cunningham1983decomposition}) functions can be viewed as \emph{information} functions:
\begin{align}
  I_f(A) = f(A)  
\end{align}
since they satisfy all of the requisite Shannon inequalities \citep{yeung2008information,mcgill1954multivariate, shannon1948mathematical} that make them natural for such a purpose. 

\paragraph{Conditional gain:}
Given a set function $f: 2^V \to \mathbb R$, we define the conditional gain $f(A|B)$ as the gain in function value by adding $B$ to $A$:
\begin{align}
H_f(A | B) = f(A|B) \triangleq f(A \cup B) - f(B)    
\end{align}
When $f(A) = H(X_A)$, $f(A | B)$  corresponds to the conditional entropy: $H(X_A | X_B)$.\looseness-1

\paragraph{Submodular (Conditional) Mutual Information:} We define the submodular mutual information between two sets $A,B$ denoted by $I_f(A;B)$ as: 
\begin{align}
    I_f(A;B) \triangleq f(A) + f(B) - f(A \cup B)
\end{align}
It is again easy to see that $I_f(A; B)$ is equal to the mutual information between two random variables when $f$ is the entropy function. Note that $I_f(A;B) = f(A) - f(A|B)$. Also, define the submodular conditional mutual information of sets $A,B$ given set $C$ as 
\begin{align}
    I_f(A; B | C) &\triangleq f(A | C) + f(B | C) - f(A \cup B | C) \nonumber \\ 
    &= f(A \cup C) + f(B \cup C) - f(A \cup B \cup C) - f(C)
\end{align}

\paragraph{Submodular Multi-Set Mutual Information and Total Correlation:} 
Define the $k$ way \emph{submodular multi-set mutual information} $I_f(A_1; A_2; \dots; A_k)$, with $A_1, A_2, \dots, A_k \subseteq \Omega$ as:
\begin{align}
    I_f(A_1; A_2;\dots; A_k) \triangleq -\sum_{T \subseteq [k]} (-1)^{|T|} f(\cup_{i \in T} A_i)
\end{align}
which is defined via the principle of inclusion-exclusion. We  also define the submodular conditional $k$-set mutual information as 
\begin{align}
    I_f(A_1; A_2; \dots; A_k|C) \triangleq -\sum_{T \subseteq [k]} (-1)^{|T|} f(\cup_{i \in T} A_i|C).
\end{align}
Finally, define the \emph{submodular total correlation} and its conditional version as:
\begin{align}
   C_f(A_1, A_2, \dots, A_k) \triangleq ( \sum_{i = 1}^k f(A_i) ) - f(\cup_{i = 1}^k A_i) \\
   C_f(A_1, A_2, \dots, A_k | C) \triangleq ( \sum_{i = 1}^k f(A_i | C) ) - f(\cup_{i = 1}^k A_i | C)
\end{align}
When $k = 2$, the (conditional) multi-set mutual information and total correlation both yield the (conditional) 2-set submodular mutual information, but they are different when $k > 2$. When $k = 1$, the multi-set mutual information gives the submodular information function.\looseness-1

\paragraph{Submodular Information Metric:} To encode the notion of a distance between sets $A$ and $B$ we define the \emph{submodular variation of information} distance measure and denote it as 
\begin{align}
  D_f(A; B) \triangleq f(A \cup B) - I_f(A; B)  
\end{align}
As we shall see in the next section, this quantity is actually \emph{pseudo-metric} for a monotone submodular function $f$. The submodular information metric also has the intuitive form:
\begin{align}
    D_f(A; B) = f(A|B) + f(B|A)
\end{align}

\section{Properties of the Submodular Information Measures}\label{sec:gen-properties}

The properties of the submodular information measures, that we introduce in this section, generally hold for all monotone submodular functions. In certain cases, interestingly and as we point out below, they hold even for relaxed versions (i.e., just submodular, or just subadditive, or just monotone). Proofs of the results here are in Appendix~\ref{sec:app-proofs-sec3}.\looseness-1

\paragraph{Submodular Information Functions: } We start with a very simple property of the submodular information functions:
\begin{lemma}\label{lemma:submodinfo-properties}
Given a monotone, non-negative submodular function $f$, it holds that $f(\cup_i A_i) = I_f(\cup_i A_i) \leq \sum_i f(A_i) = \sum_i I_f(A_i)$. Furthermore, $I_f(A) = 0$ iff $f(j) = 0, \forall j \in A$. In other words, the elements $j$ are \emph{dummy} variables with no information.
\end{lemma}

\paragraph{Submodular Conditional Gain:}
Below, we study three properties of the submodular conditional gain. The first and second are well known. The second property can be also viewed as the  ``conditioning reduces valuation'' counterpart of the conditional entropy.

\begin{lemma} \label{lemma:submodcond-gain-properties}
The conditional gain is non-negative, i.e., $f(A|B) \geq 0$ if $f$ is a monotone function. We also have the upper bound $f(A|B) \leq f(A)$ if $f$ is monotone and subadditive. Finally, $f(A | B)$ is submodular in $A$ for a given set $B$ (but not vice-versa) if $f$ is submodular. 
\end{lemma}

\paragraph{Submodular (Conditional) Mutual Information:}
Now we take a look at the basic properties of the submodular mutual information. Before getting into more details on the properties, we show one very interesting connection between the submodular mutual information and the submodular conditional mutual information.
\begin{lemma}\label{lemma:app-submod-MI-CMI-rel}
Given a monotone, non-negative and normalized submodular function $f$, the conditional submodular mutual information $I_f(A;B \ | \ C) = I_g(A;B)$ where $g(A) = f(A | C)$ is also a normalized, monotone and non-negative submodular function. In other words, the conditional submodular mutual information is equal to the submodular mutual information of the conditional function.\looseness-1
\end{lemma}
This result follows by definition since $I_f(A;B \ | \ C) = f(A | C) + f(B | C) - f(A \cup B | C) = g(A) + g(B) - g(A \cup B) = I_g(A; B)$. While this is a very simple result, the take-away is very interesting. This means that whatever results hold for the submodular mutual information, will also hold with the conditional mutual information since it is in fact the submodular mutual information parameterized with a different submodular function.

Next, we list several interesting properties of SMIs and CSMIs. All the properties shown below also hold for the mutual information $I(X_A; X_B)$ as a function over sets of random variables $A$ and $B$. Two properties which follow directly by definition are: a) \emph{symmetry}: $I_f(A;B) = I_f(B;A)$, and b) \emph{self-information}: $I_f(A;A) = f(A)$. The next lemma shows the non-negativity of (conditional) submodular mutual information and provides upper and lower bounds.\looseness-1

\begin{lemma} \label{lemma:submod-mutinfo-basic-prop}
When $f$ is submodular, $I_f(A;B) \geq 0$ and $I_f(A;B|C) \geq 0$. Also: $\min(f(A), f(B)) \geq I_f(A;B) \geq f(A \cap B)$ and $\min(f(A|C), f(B|C)) \geq I_f(A;B | C) \geq f(A \cap B | C)$.\looseness-1
\end{lemma}

\noindent Next, we study $I_f(A; B)$ (and $I_f(A; B | C)$ for a fixed set $B$ (and $C$). We can then view $I_f(A;B)$ (or $I_f(A; B | C)$) as a function of $A$. The following theorem studies their monotonicity and submodularity. 
\begin{theorem} \label{theorem:submod-mutinfo-mon-submod-prop}
For a given sets $B$ and $C$, $I_f(A;B)$ and $I_f(A; B | C)$ are monotone functions in $A$, though in general, they are neither submodular nor supermodular. They are submodular, however, iff the second-order partial derivatives, $f^{(2)}(j,k;A)$ are monotone increasing or equivalently the 
third-order partial derivatives $f^{(3)}(i,j,k;A) = f^{(2)}(j,k;A \cup i) - f^{(2)}(j,k;A)$ are always non-negative.\looseness-1
\end{theorem}

Recall that for monotone functions, the first-order partial derivatives $f(j | X) \geq 0$, while for submodularity, we require the second-order partial derivatives $f^{(2)}(j,k;A) \leq 0$. 
A subclass of submodular functions which additionally have non-negative third-order partial derivatives satisfy the property that the submodular (conditional) mutual information function is submodular. Such functions are also called second order supermodular functions~\citep{korula2018online}.  In the \arxivalt{Appendix~\ref{sec:app-properties-smi}}{extended version~\citep{neurips2020supplemental}}, we show that several practically useful functions like set cover, facility location and concave over modular with power and log functions satisfy this property. However we also show that a rather simple function like the uniform matroid rank function (i.e., a truncation of the form $f(A) = \min(|A|,c)$) does not satisfy this, and correspondingly, the (conditional) mutual information parametrized by these functions is not submodular. Also, the standard mutual information is not necessarily submodular in one set given the other~\citep{krauseNearoptimalNonmyopicValue}. Next, we provide simple submodular upper and lower bounds for the submodular mutual information.\looseness-1 

\begin{lemma} \label{lemma:submod-mutinfo-modularbounds-prop}
For a given set $B$, $I_f(A; B)$ can be upper/ lower bounded by two submodular functions: $f(A) - \sum_{j \in A \backslash B} f(j | B) \leq I_f(A; B) = f(A) - f(A | B) \leq f(A) - \sum_{j \in A \setminus B} f(j | \Omega \setminus j) \leq f(A)$.
\end{lemma}


\paragraph{Submodular Multi-Set Mutual Information and Total Correlation:}
We begin this section by studying the positivity and monotonicity of the $k$-way total correlation.
\begin{lemma}
\label{lemma:pos-monotone-tc-prop}
Given a monotone submodular function, it holds that $C_f(A_1, \cdots, A_k) \geq 0$. Furthermore, the $k$-way total correlation $C_f(A_1; \cdots; A_k)$ is monotone in any one of the variables fixing the others. Correspondingly, it is also monotone in all the variables. 
\end{lemma}
While $I_f(A; B)$ and the total correlation are always non-negative, the $k$-way mutual information might be negative even for $I_f(A;B;C)$ when $k = 3$ (similar to what \citep{yeung2008information} showed in the entropic case).\looseness-1
\begin{lemma}
\label{lemma:positivity-mi-prop}
The $3$-way submodular mutual information $I_f(A; B; C)$, can be both positive or negative. $I_f(A; B; C) \geq 0$ if $I_f(A; B)$ is submodular in $A$ for a fixed set $B$\arxiv{ (equivalently $A, B, C$ are redundant with respect to each other)}. Similarly, $I_f(A; B; C) \leq 0$ if $I_f(A; B)$ is supermodular in $A$ for a fixed set $B$\arxiv{ (i.e. sets $A, B, C$ are synergistic with respect to each other)}. 
\end{lemma}
 In the \arxivalt{Appendix~\ref{sec:app-properties-multiset-tc-smi}}{extended version~\citep{neurips2020supplemental}}, we provide examples of $I_f(A;B;C) < 0$. Since the three-way mutual information is not necessarily non-negative, we do not expect the $k$-way mutual information to be non-negative in general. The next result discusses the monotonicity of the three-way mutual information.

\begin{lemma}\label{lemma:3way-mutinfo-monotone}
For a given set $D$ and a function $f$, we have $I_f(A\cup D;B\cup D;C\cup D) \geq I_f(A;B;C)$ for any three subsets $A,B,C$. On the other hand with fixed sets $B, C$ we do not always have monotonicity in $I_f(A;B;C)$ as function of $A$. 
\end{lemma}
Again, contrasting this with the two-set case, $I_f(A;B)$ is monotone in $A$ for a given $B$ and vice versa. However, this does not hold for the three-way mutual information. Moreover, for the four-way mutual information does not satisfy the above (i.e. $I_f(A \cup E; B \cup E; C \cup E; D \cup E)$ may be smaller or larger compared to $I_f(A; B; C; D)$, and hence we do not expect the $k$-way submodular mutual information to be monotone (cf.\ \arxivalt{Appendix~\ref{sec:app-properties-multiset-tc-smi}}{extended version~\citep{neurips2020supplemental}}). Finally, we provide an upper bound on the mutual Information:\looseness-1
\begin{lemma}
\label{lemma:upperbound-mutinfo-prop}
Given a monotone submodular function, the following inequality: \\ $I_f(A_1; A_2; \cdots; A_k) \leq \min(f(A_1), \cdots, f(A_k))$ holds for $k = 3$ and $4$. However, it does not hold for $k = 5$, and hence does not necessarily hold for $k \geq 5$.
\end{lemma}

\paragraph{Submodular Information Metric:}
We next show that the submodular information distance measure is a \emph{pseudo metric} for any general submodular function $f$ and then we state the condition on $f$ for $D_f(A;B)$ to be a metric between subsets of $\Omega$. 
\begin{lemma}
\label{lemma:psuedometric-def}
Given a monotone submodular function $f$, the submodular information metric: $D_f(A,B) = f(A \cup B) - I_f(A;B)$ is a pseudo metric i.e $D_f(A;B) = 0 \centernot \implies A = B$. Moreover, it is a metric if the submodular function has a curvature $\kappa_f > 0$ where $\kappa_f =  1 - \min_{j \in \Omega} \frac{f(j | \Omega\backslash j)}{f(j | \emptyset)}$
\end{lemma}
Curvature defined above is a widely used notion in quantifying approximation bounds for submodular optimization~\citep{conforti1984submodular, iyerCurvatureOptimalAlgorithms2013a, vondrak2010submodularity}. Next, we provide upper and lower bounds to this metric in terms of the submodular Hamming metric and its additive version \citep{gillenwaterSubmodularHammingMetrics}. The submodular Hamming metric is defined as $D^{SH}_f(A, B) = f(A \Delta B) = f(A \backslash B \cup B \backslash A)$ 
and its additive version as: $D^{SHA}_f(A, B) = f(A \backslash B) + f(B \backslash A)$. We define the \emph{curvature} at a set $A$ as $\kappa_f(A) = 1 - \min_{j \in A} \frac{f(j | A \setminus j)}{f(j | \emptyset)}$. Note that $\kappa_f$ defined previously corresponds to this definition via $ \kappa_f = \kappa_f(\Omega)$.\looseness-1
\begin{lemma}
\label{lemma:psuedometric-bound}
Given a monotone submodular function $f$ and two sets $A, B$, it holds that: $(1 - \kappa_f(A \cup B)) D^{SH}(A, B) \leq (1 -\kappa_f(A \cup B)) D^{SHA}(A, B) \leq D_f(A, B) \leq D^{SH}(A, B) \leq D^{SHA}(A, B)$.
\end{lemma}

\begin{table}[!ht]
\scriptsize{
\begin{tabular}{|l|l|l|l|}
\hline
Function          & $f(A)$                                                   & $f(A|B)$                                                                                                & $D_f(A,B)$                                                                                         \\ \hline
Modular           & $\sum\limits_{i \in A} w(i) = w(A)$                                    & $w(A \setminus B)$                                                                       & $w(A \Delta B)$                                                                     \\ 
Facility Location & $\sum\limits_{i \in \Omega} \underset{a \in A}{\max} \; s_{ia}$ & $\sum\limits_{i \in \Omega} \max(0, \underset{a \in A}{\max} \; s_{ia} - \underset{b \in B}{\max} \;  s_{ib})$ & $\sum\limits_{i \in \Omega} | \underset{a \in A}{\max} \;  s_{ia} - \underset{b \in B}{\max} \;  s_{ib}|$ \\ 
Set Cover         & $w(\gamma(A))$                                           & $w(\gamma(A) \setminus \gamma(B))$                                                                      & $w(\gamma(A) \Delta \gamma(B))$                                                                    \\ 
Prob. Set Cov. & $\sum\limits_{i \in U} w_i P_i^{\complement}(A)$ & $\sum\limits_{i \in U} w_i P_i^{\complement}(A) P_i(B)$ & $\sum\limits_{i \in U} w_i [P_i^{\complement}(A)P_i(B) + P_i^{\complement}(B)P_i(A)]$ \\ \hline
\end{tabular}
\caption{\small{Instantiating $f(A), f(A|B)$ and $D_f(A, B)$ with different $f$}}}
\label{table:submod-cgdistsummary}
\end{table}

\begin{table}[!ht]
\scriptsize
\begin{tabular}{|l|l|l|}
\hline
Function          & $I_f(A;B)$                                                                                         & $I_f(A_1; \dots ; A_k)$                                                                                              \\ 
Modular           & $w(A \cap B)$                                                                                      & $w(\cap_{i=1}^k A_i)$                                                                                                \\ 
Facility Location & $\sum_{i \in \Omega} \min(\underset{a \in A}{\max} \; s(i,a), \underset{b \in B}{\max} \; s(i,b))$ & $\sum_{i \in \Omega} \min(\underset{a_1 \in A_1}{\max} \; s(i,a_1) ; \dots; \underset{a_k \in A_k}{\max} \; s(i,a_k))$ \\ \hline
Set Cover         & $w(\gamma(A) \cap \gamma(B))$                                                                      & $w(\cap_{i = 1}^k \gamma(A_i))$                                                                                      \\ 
Prob. Set Cov. & $\sum_{i \in U} w_i  P_i^{\complement}(A) P_i^{\complement}(B)$ & $\sum_{i \in U} w_i  \Pi_{j = 1}^k P_i^{\complement}(A_j)$  \\ \hline

\end{tabular}
\caption{\small{Instantiating $I_f$ and multi-set $I_f$ with different $f$}}
\label{table:submod-gen-mutinfo-summary}
\end{table}

\section{Examples of Submodular Information Measures} \label{sec:examples}
 We instantiate the submodular information measures with modular functions, set-cover, probabilistic set cover, the facility location function, and the graph cut function. We summarize the formulations of the submodular information measures for some representative functions in Tables 1 and 2. The proofs of the results here are in Appendix~\ref{sec:app-proofs-sec4}.

\paragraph{Modular Function:} In this case the function $f$ is modular i.e $f(A) = w(A) = \sum_{a \in A} w(a)$, for some weight vector $w$ over the elements in the ground set $\Omega$.
\begin{lemma}
\label{lemma:def-modular}
If $f(A) = w(A)$ is a modular function, then $I_f(A;B) = w(A \cap B)$, $f(A|B) = w(A \setminus B)$, and $D_f(A,B) = w(A \Delta B)$. Similarly, $I_f(A_1; \dots; A_k) = w(\cap_{i=1}^k A_i)$. Finally, $A \perp_f B$ iff $A$ and $B$ are disjoint, and $A \perp_f C \; | \; B$ iff $A \cap C \subseteq B$ 
\end{lemma}
Some interesting observations are that the submodular information metric is the weighed hamming distance. Similarly, $I_f(A; B) = w(A \cap B)$ is a modular function in one argument given the other.

\paragraph{Weighted Set Cover:} \label{lemma:def-setcover}
Here $f$ is a weighted set cover function, $f(A) = w(\cup_{a \in A} \gamma(a)) = w(\gamma(A))$ where $w$ is a weight vector in $\mathbb{R}^{\gamma(\Omega)}$. Intuitively, each element in $\Omega $ \emph{covers} a set of elements from the concept set $U$ and hence $w(\gamma(A))$ is total weight of concepts covered by elements in $A$. Note that $\gamma(A \cup B) = \gamma(A) \cup \gamma(B)$ and hence $f(A \cup B) = w(\gamma(A \cup B)) = w(\gamma(A) \cup \gamma(B))$. 
\begin{lemma}
\label{lemma:setcover}
When $f(A) = w(\gamma(A))$ is the set cover function, $I_f(A;B) = w(\gamma(A) \cap \gamma(B))$ 
and $f(A|B) = w(\gamma(A) \setminus \gamma(B))$ 
Similarly, $D_f(A,B) = w(\gamma(A) \setminus \gamma(B)) + w(\gamma(B) \setminus \gamma(A)) = w(\gamma(A) \Delta \gamma(B))$. Finally, the multi-set mutual information is $I_f(A_1; \dots; A_k) = w(\cap_{i = 1}^k \gamma(A_i))$. 
\looseness-1
\end{lemma}
Observe that with the set cover, $I_f(A;B)$ will be large if $A$ and $B$ cover similar concepts. A similar form is seen with the $k$ way mutual information. 
Also, the submodular information metric between $A$ and $B$ is the weighted hamming distance between the respective covered sets $\gamma(A) \Delta \gamma(B)$.  

\paragraph{Probabilistic Set Cover: }  
Here $f$ is the Probabilistic Set Cover function, $f(A) = \sum_{i \in U} w_i (1 - \Pi_{a \in A} (1-p_{ia}))$. $p_{ia}$ represents the probability that the element $a \in A$ covers the concept $i \in U$, where $U$ is the set of all concepts. The probabilistic set cover is a soft generalization of Set Cover function and in the case where the probabilities are fixed to be only either $0$ or $1$, we recover the formulation for the non-probabilistic counterpart. 
We denote by $P_i(A) = \Pi_{a \in A} (1-p_{ia})$, the probability that none of the elements in $A$ cover the concept $i$. Hence, $1 - P_i(A)$ denotes that \emph{at least} one element in $A$ covers the concept $i$. With the above notation, we have the following result:
\begin{lemma} \label{lemma:def-probsetcover}
With $f(A) = \sum_{i \in U} w_i (1 - P_i(A))$ as the Probabilistic Set Cover function, we have that $I_f(A;B) =  \sum_{i \in U} w_i (1 - (P_i(A) + P_i(B) - P_i(A\cup B)))$. When $A$ and $B$ are disjoint, $I_f(A;B) =  \sum_{i \in U} w_i (1 - P_i(A)) (1 - P_i(B))$. Similarly, $f(A|B)  = \sum_{i \in U} w_i P_i(B) (1 - P_i(A \setminus B))$ and $D_f(A,B) = \sum_{i \in U} w_i [P_i(B)(1 - P_i(A \setminus B)) + P_i(A) (1 - P_i(B \setminus A))]$. 
\end{lemma}
Note that when $A$ and $B$ are disjoint, $I_f(A;B) =  \sum_{i \in U} w_i (1 - P_i(A)) (1 - P_i(B))$ which essentially is the probability that sets $A$ and $B$ both cover $i\in U$. Similarly with $A_1,\dots,A_k$ being pairwise disjoint we have the multi set mutual information $I_f(A_1,\dots,A_k) =\sum_{i \in U} w_i \Pi_{j=1}^k (1 - P_i(A_j)$. With $f(A|B)$ we have that for a concept $i$, its term adds to the value only if $A \setminus B$ covers $i$ and at the same time $B$ does not, both with non-zero probability. Consequently the distance metric looks like the symmetric version of the previous observation by the virtue of its definition: $D_f(A;B) = f(A|B) + f(B|A)$. Another interesting observation is that when $A$ and $B$ are disjoint, $I_f(A; B)$ is submodular in $A$ for a given $B$ can be thought of another instance of the Probabilistic Set Cover function with the weight for each concept $i$ being $w_i (1 - P_i(B))$ instead of just $w_i$.

\paragraph{Facility Location:}
Here $f$ is a facility location function, $f(A) = \sum_{i \in \Omega} \max_{a \in A} s(i,a)$ where $s$ is similarity kernel between the items in $\Omega$ such that the similarity between identical points is highest and equal to 1. In essence the value $f(A)$ models how representative the set $A$ is for the ground set $\Omega$. 
\begin{lemma}
\label{lemma:facloc}
With $f(A) = \sum_{i \in \Omega} \max_{a \in A} s_{ia}$, $I_f(A;B) = \sum_{i \in \Omega} \min(\max_{a \in A} s_{ia}, \max_{b \in B} s_{ib})$. Also, $f(A|B) = \sum_{i \in \Omega} \max(0, \max_{a \in A} s_{ia} - \max_{b \in B} s_{ib})$ and $D_f(A,B) = \sum_{i \in \Omega} | \max_{a \in A} s_{ia} - \max_{b \in B} s_{ib}|$. Finally, $I_f(A_1; \dots; A_k) = \sum_{i \in \Omega} \min(\max_{a_1 \in A_1} s_{ia_1}, \dots, \max_{a_k \in A_k} s_{ia_k})$. 
\end{lemma}
The mutual information, conditional gain and variation of information metric all have intuitive expressions. In particular, the mutual information is a truncated facility location function (where the truncation depends on how much $B$ represents the items $i \in \Omega$, and hence is submodular), and the distance metric is an absolute difference between the how well sets $A$ and $B$ represent each item $i$. 

\paragraph{Generalized  Graph Cut: } Here $f$ is the generalized graph cut function, $f(A) = \lambda \sum_{i \in \Omega} \sum_{a \in A} s_{i a} - \sum_{a_1,a_2 \in A} s_{a_1 a_2} $ with $\lambda \geq 2$ and $s$ as a similarity kernel. Note that the condition on $\lambda$ is to ensure that $f$ remains a monotone submodular function.
\begin{lemma}
\label{lemma:def-graphcut}
With $f(A) = \lambda \sum_{i \in \Omega} \sum_{a \in A} s_{i a} - \sum_{a_1,a_2 \in A} s_{a_1 a_2} $ as the generalized graph cut function, we have
$I_f(A;B) = f(A \cap B) + 2 \sum_{a\in A, b \in B} s_{ab}
- 2 \sum_{c \in A\cup B, d \in A \cap B} s_{cd}$. When $A,B$ are disjoint, it follows that  $I_f(A;B) =  2 \sum_{a \in A} \sum_{b \in B} s_{a b}$ which is the cross-similarity measure for sets $A,B$. We also have $ f(A|B) = f(A \setminus B) -  2 \sum_{a' \in A \setminus B} \sum_{b \in B} s_{a' b}$.  
\end{lemma}
The submodular mutual information (in the case of disjoint sets $A$ are $B$) is intuitive since higher the pairwise similarity sum, higher the mutual information. The conditional gain also reduces to $f(A|B) = f(A) -  2 \sum_{a \in A} \sum_{b \in B} s_{a b}$, which again makes sense -- higher the cross similarity between $A$ and $B$, lower is the conditional gain of adding $A$ to $B$.  


\section{Optimization Problems and Applications} \label{sec:optimization-problems}
In this section we study a number of applications of the information theoretic quantities introduced in Section~\ref{sec:def-gen}, along with a few different optimization problems.  


\paragraph{Submodular Mutual Information Based Selection: }
The first problem we consider is the maximization of the submodular information between a subset and its complement under a cardinality constraint, i.e. maximizing $I_f(A; \Omega \backslash A)$ in $A$ under cardinality constraints. This generalizes the symmetric mutual information maximization for sensor placement~\citep{krauseNearOptimalSensorPlacementsa} and maximum graph cut under cardinality constraints. We argue that compared to just optimizing the information $f(A)$, solving $\max_{A \subseteq \Omega, |A| \leq k} I_f(A; \Omega \backslash A)$ is less sensitive to outliers since we are also enforcing that the subset $A$ be similar to $\Omega \setminus A$. For a concrete example, consider the set cover function where we have $I_f(A; \Omega \setminus A) = w(\gamma(A) \cap \gamma(\Omega \setminus A))$. Consider the case when there are outlier concepts present in the data. Trying to maximize $f$ directly will result in a subset that is likely to also cover the outlier concepts since that will increase the functional value of set $f(A) = w(\gamma(A))$. However with $I_f(A; \Omega \setminus A)$ if the subset $A$ covers outliers, it is likely that $\Omega \setminus A$ will not cover the same, leading to a lower functional value due to the intersection between the two and hence, in the process, it will discourage choosing such a subset. Note that $I_f(A; \Omega \setminus A)$ is not a monotone submodular function (but it is still submodular) and hence the $(1 - \frac{1}{e})$ approximation guarantee by the greedy algorithm \citep{nemhauser1978analysis} is not applicable due to non-monotonicity. It is however an instance of cardinality constrained non-monotone submodular maximization and we can achieve a $1/e$ approximation using the randomized greedy algorithm from~\citep{buchbinderSubmodularMaximizationCardinality2014a}.\looseness-1

\begin{lemma} \label{lemma:complement-mutinfo-approx-monotone}
If we assume that $f(j) \leq 1 \; \forall j \in \Omega$, then $g(A) = I_f(A;\Omega \setminus A)$ is approximately monotone for a subset $A$ with the factor $\kappa_f(A)$, i.e. $g(j | A) \geq -\kappa_f(A) \; \forall j \in \Omega, \; A \subseteq \Omega$, where $\kappa_f(A) = max_{j \in \Omega \setminus A} \frac{f(j|V \setminus (A \cup j)}{f(j)}$. The greedy algorithm is guaranteed to give us a subset $\hat{A}$ with $k$ elements such that $I_f(\hat{A}; \Omega \setminus \hat{A}) \geq (1 - \frac{1}{e}) (OPT - k \kappa_f(A^*))$ where $A^*$ is the optimal set. 
\end{lemma}

\paragraph{Query Based and Privacy Preserving Summarization: }
Next, we consider the problem of maximizing the mutual information with respect to a \emph{fixed query set} $Q$. An application of this is query based summarization. We consider the following general optimization problem: 
\begin{align}
   \max_{A \subseteq \Omega, |A| \leq k} I_f(A; Q) + \lambda g(A) 
\end{align}
where $g(A)$ is another  submodular function modeling diversity/representation. This optimization problem (which we call submodular mutual information maximization or SMIMax) tries to trade-off representation/diversity (through $g(A)$) and query closeness (through $I_f(A; Q)$). Recall from the properties discussed in Section~\ref{sec:def-gen} (specifically, Theorem \ref{theorem:submod-mutinfo-mon-submod-prop}), that though $I_f(A;B)$ is monotone function in $A$ for a given $B$ it is not necessary submodular. 
\begin{theorem} \label{lemma-smimax-approxfactor}
If $f$ is second order supermodular, i.e. $f$ satisfies $f^{(3)}(i,j,k;A) \geq 0$ and $g$ is monotone submodular, the greedy algorithm achieves a $1 - 1/e$ approximation for SMIMax. There exists no polynomial time approximation algorithm for SMIMax. Specifically, with $n$ being the size of the problem instance and $\alpha(n) > 0$ to be a positive poly-time computable function of $n$, there cannot exist a polynomial time algorithm which is guaranteed to find a subset $\hat{A}$ ($|\hat{A}| \leq k$) such that $I_f(\hat{A};B) \geq \alpha(n) OPT$ with $OPT = \max_{A \subseteq \Omega, |A| \leq k} I_f(A;B)$.
\end{theorem}
Several useful submodular functions like facility location and set cover satisfy this. 
Unlike query based summarization, privacy preserving summarization's  goal is to select a summary $A$ that has low submodular mutual information to a set $P$ which is private information the summary should not represent. Following the above, we can pose this as $\max_{A \subseteq \Omega, |A| \leq k} \lambda g(A) - I_f(A; P) = \max_{A \subseteq \Omega, |A| \leq k} \lambda g(A) + f(P | A)$. We call this NSMIMax, since it involves maximizing the negative of submodular mutual information plus a submodular function. Unfortunately, NSMIMax is not tractable (in most cases, as we show below), and hence we modify the objective to maximize the conditional gain $f(A | P)$ instead of maximizing $f(P | A)$. Note that this is a \emph{different} optimization problem, but it tries to achieve a similar goal, i.e., obtain a set $A$ as \emph{different as possible} from $P$. The optimization problem is (we call it CGMax): $\max_{A \subseteq \Omega, |A| \leq k} \lambda g(A) + f(A | P)$. 
\begin{lemma}\label{lemma-approx-nsmimax-cgmax}
NSMIMax is an instance of (non-monotone) submodular maximization if $f$ is second-order submodular, i.e. $f$ satisfies $f^{(3)}(i,j,k;A) \leq 0$. In the worst case, there exists submodular function $f$ such that NSMIMax is inapproximable. However, CGMax is an instance of monotone submodular maximization if $f$ and $g$ are monotone, and correspondingly the greedy algorithm admits a $1 - 1/e$ approximation.\looseness-1
\end{lemma}
For many subclasses of functions, the third-order partial derivatives are non-negative instead of non-positive and hence NSMIMax is typically a non-monotone difference of submodular functions, and hence a much harder problem. However, CGMax is always monotone submodular maximization and for this reason preferred. 

Finally, we consider the scenario of simultaneous query based and privacy preserving summarization. We formulate this in two ways. The first is to solve: $\max_{A \subseteq \Omega, |A| \leq k} I_f(A; Q) - I_f(A; P) + \lambda g(A)$. Unfortunately, this is again a difference of submodular functions and not approximable. The second formulation is 
\begin{align}
    \max_{A \subseteq \Omega, |A| \leq k} I_f(A; Q | P)  + \lambda g(A)
\end{align}
or in other words maximize the \emph{conditional submodular mutual information}. We call this problem: CSMIMax. This formulation  (as we argue in \arxivalt{Appendix~\ref{sec:app-proofs-sec5}}{the extended version}), tries to obtain a set $A$ that is similar to $Q$ yet independent of $P$. The following Lemma provides the approximation bound for CSMIMax.
\begin{lemma}\label{lemma-approx-csmimax}
CSMIMax is an instance of monotone submodular maximization if the third-order partial derivatives are not negative, i.e. $f^{(3)}(i,j,k;A) = f^{(2)}(j,k;A \cup i) - f^{(2)}(j,k;A) \geq 0$. As a result, the greedy algorithm~\citep{nemhauser1978analysis} achievea a $1 - 1/e$ approximation . 
\end{lemma}
What is also interesting is that CSIMax yields SMIMax and CGMax when considering only query or privacy preserving summarization. In particular, if $Q = V, I_f(A; V | P) = f(A | P)$ which is CGMax. Similarly, when $P = \emptyset, I_f(A; Q | \emptyset) = I_f(A; Q)$ which is SMIMax. Furthermore, when both $P = \emptyset, Q = V, I_f(A; V | \emptyset) = f(A)$ (which is query and privacy agnostic summarization). \looseness-1

\paragraph{Clustering and Partioning using Multi-Set Information measures:} 
We end this section by studying the problem of maximizing and minimizing the $k$-way multi-set information measures (multi-set submodular mutual information and multi-set submodular total correlation) such the sets $A_1, \cdots, A_k$ form a partition of $\Omega$ (i.e., for all $i, j$, $A_i \cap A_j = \emptyset$ and $\cup_{i = 1}^k A_i = \Omega$). An application of the minimization problem is clustering~\citep{narasimhan2006q} while the maximization problem can be applied to diverse partitioning~\citep{weiMixedRobustAverage2015}. Maximizing and minimizing the submodular total correlation over partitions turn out to be related to two well studied problems.

\begin{observation} 
Minimizing the $k$-way submodular total correlation is equivalent the submodular multi-way partition \citep{zhaoGreedySplittingAlgorithms2005, chekuriApproximationAlgorithmsSubmodular2011} minus a constant. Similarly, maximizing the $k$-way total correlation is the submodular is exactly the submodular welfare problem \citep{feigeSubmodularWelfareProblem, vondrakOptimalApproximationSubmodular2008} minus a constant.
\end{observation}

\noindent In \arxivalt{Appendix~\ref{sec:app-proofs-sec5}}{the extended version}, we study the multi-set mutual information partitioning problems. In particular, we show that minimizing the multi-set submodular mutual information does not make sense and for clustering, the $k$-way total correlation is better. In the maximization setting though, the $k$-way submodular mutual information is related to \emph{robust partitioning}~\citep{weiMixedRobustAverage2015}, while the total correlation is the average partitioning. \looseness-1

\paragraph{Minimization of the Submodular Information Metric: } Finally we take a look at the centroid finding problem in the context of the submodular information metric in an unconstrained setting. We want to find a subset $A$ over all subsets of $\Omega$ such that it is as similar as possible to a collection $S_1, S_2, \dots, S_m$ with respect to the submodular information metric: 
\begin{align}\label{opt:min-info-metric}
\min_{A \subseteq \Omega} \sum_{i=1}^{m} D_f(A, S_i)   
\end{align}
This problem is related to problem of minimizing the submodular hamming metric (and the additive submodular hamming metric) as shown in Lemma~\ref{lemma:psuedometric-bound}. We use this to show the following approximation guarantee for Equation~\eqref{opt:min-info-metric}.

\begin{lemma}\label{psudeometric-opt-lemma}
We can approximately solve the problem in equation \ref{opt:min-info-metric} with the approximation guarantee of $1-\kappa_f$, where $\kappa_f$ is the worst case curvature of $f$.
\end{lemma}
\noindent The proof of this result is in Appendix~\ref{sec:proof-psudeometric-opt-lemma}.

\section{Conclusion}
 In conclusion, we study a number of submodular information measures, their properties and instantiations on a number of common submodular functions. We also investigate a number of optimization problems related to these information measures on data summarization. In future work, we would like to study learning problems with these information measures, experimentally validate the proposed algorithms on real and synthetic data, and also apply them to other problems such as data subset selection and feature selection. For future work, we will study applications of the submodular information measures, specifically, the submodular mutual information, submodular conditional gain, submodular multi-set mutual information and submodular total correlation to applications such as summarization, targeted subset selection, clustering, disparate partitioning, and diverse $k$-best MAP inference.
 

\bibliography{references}

\arxiv{
\appendix

\section{Proofs of the Results from Section~\ref{sec:gen-properties} and Additional Results}
\label{sec:app-proofs-sec3}
In this section, we prove the Lemmas introduced in Section~\ref{sec:gen-properties} and also some results we did not cover in Section~\ref{sec:gen-properties} due to space limitations.

\subsection{Submodular Information Functions: Proofs}
Proof of Lemma~\ref{lemma:submodinfo-properties}
\begin{proof}
The first part of the proof follows directly from the subadditivity of $f$, since $f$ is submodular. The second part follows from the very simple observation that $I_f(A) \geq f(j), \forall j \in A$ and since $f(j) \geq 0$ (because of non-negativity) and $I_f(A) = 0$, this implies that $f(j) = 0, \forall j \in A$.
\end{proof}
We contrast the second with the result of entropy where $H(X) = 0$ iff $X$ is a deterministic variable. In the case of combinatorial information functions, it means the variables $j \in A$ have no information. Also note that the proof has only used monotonicity and non-negativity of the information functions.

\subsection{Submodular Conditional Gain: Proofs}
\label{sec:app-properties-cg}
\subsubsection{Proof of Lemma \ref{lemma:submodcond-gain-properties} (Properties of Conditional Gain)}
\begin{proof}
If $f$ is monotone, we have: 
\begin{align}
f(A \cup B) \geq f(B) \; \forall A,B \subseteq \Omega \implies f(A|B) \geq 0.
\end{align}
If $f$ is subadditive then, 
\begin{align}
    f(A) + f(B) \geq f(A \cup B) \implies f(A) \geq f(A \cup B) - f(B) = f(A|B) \; \forall A,B \subseteq \Omega.
\end{align} 
Finally, if $f$ is submodular, $f(A|B) = f(A \cup B) - f(B)$ will be a submodular function in $A$ as $f(A \cup B)$ is submodular in $A$ for a fixed set $B$ and $f(B)$ is just a constant. Note that however, $f(A | B) = f(A \cup B) - f(B)$ is not submodular in $B$ for a given $A$.
\end{proof}

\subsubsection{Other Results on Conditional Gain}
Similar to the submodular information functions, we study the case when the conditional $f(A | B) = 0$.
\begin{lemma}
Given a monotone, non-negative submodular function $f$, it holds that $f(A | B) = 0$ iff $f(j | B) = 0, \forall j \in A$. In other words, given $B$, every element $j \in A$ has no additional information.
\end{lemma}
\begin{proof}
Again, invoking monotonicity, we have $f(B) = f(A \cup B) \geq f(j \cup B), \forall j \in A$. Also, because of monotonicity, we have that $f(j \cup B) \geq f(B)$ which implies that $f(B \cup j) = f(B), \forall j \in A$ which implies that $f(j | B) = 0, \forall j \in A$.
\end{proof}
Again, contrasting with random variables, note that $H(Y | X) = 0$ iff $Y$ is a deterministic function of $X$, or in other words, $Y$ has no extra \emph{information} on top of the information that $X$ has. In the case of combinatorial information functions, this means that the elements $j \in A$ have no extra information over and above the information contained in set $B$. Also note that the proof has only used monotonicity of the information functions.

Finally, we study properties of the conditional gain that follow directly from definition.
\begin{lemma}
Given a submodular function $g(A) = f_1(A) + f_2(A)$, $g(A | B) = f_1(A | B) + f_2(A | B)$. Furthermore, if $f(A) = c$ (a constant), then $f(A | B) = 0$, for all sets $A, B \subseteq \Omega$. FInally, if $g(A) = \lambda f(A)$, $g(A | B) = \lambda f(A | B)$.
\end{lemma}


\subsection{Submodular Mutual Information: Proofs}
\label{sec:app-properties-smi}
\subsubsection{Proof of Lemma \ref{lemma:submod-mutinfo-basic-prop} (Non-Negativity and Bounds of $I_f(A; B)$)}
\begin{proof}
The non-negativity of $I_f(A;B)$ follows from the definition of a subadditive function since $I_f(A;B) = f(A) + f(B) - f(A \cup B)$ and for a subadditive function we have 
\begin{align}
f(A) + f(B) - f(A \cup B) \implies I_f(A;B) \geq 0.    
\end{align}
Similarly, for the conditional mutual information note that 
\begin{align}
I_f(A;B|C) &= f(A|C) + f(B|C) - f(A \cup B | C) \nonumber \\ 
        &= f(A \cup C) + f(B \cup C) - f(A \cup B \cup C) - f(C)    
\end{align}
Note that thanks to the submodularity of $f$, we have that: 
\begin{align}
f(A \cup C) + f(B \cup C) &\geq f(A \cup B \cup C) + f([A \cup C] \cap [B \cup C]) \nonumber \\
& = f(A \cup B \cup C) + f([A \cap B] \cup C) \nonumber \\
&\geq f(A \cup B \cup C) + f(C)    
\end{align}
The last inequality follows from the monotonicity of $f$. Hence when $f$ is monotone submodular, the conditional mutual information is non-negative.

For the lower bound, we use the  submodularity of $f$: $f(A) + f(B) \geq f(A \cup B) + f(A \cap B) \implies I_f(A;B) = f(A) + f(B) - f(A \cup B) \geq  f(A \cap B)$. 
For the upper bound, we rewrite the definition as $I_f(A;B) = f(A) - f(A|B) = f(B) - f(B|A)$ and since $f(A|B) , f(B|A) \geq 0 \; \forall A,B \subseteq \Omega$ for a monotone submodular function (from Lemma \ref{lemma:submodcond-gain-properties}) this gives us $I_f(A;B) \leq f(A)$ and $I_f(A;B) \leq f(B) \implies I_f(A;B) \leq \min(f(A), f(B))$. 

The upper and lower bounds for conditional mutual information similarly hold, just replacing $f(A)$ and $f(B)$ by $f(A | C)$ and $f(B | C)$ and using Lemma~\ref{lemma:app-submod-MI-CMI-rel}.
\end{proof}

\subsubsection{Proof of Theorem \ref{theorem:submod-mutinfo-mon-submod-prop} (Monotonicity and Submodularity of $I_f(A; B)$)}
We first prove Theorem~\ref{theorem:submod-mutinfo-mon-submod-prop}.
\begin{proof}
Consider $g(A) = I_f(A;B)$ as a function of $A$ with a fixed set $B$. Now we consider the gain of adding an element $j \notin A\cup B$ to $A$ i.e $g(j|A) = f(j|A) - f(j|A\cup B)$. Since $f$ is a submodular function we will have $ f(j|A) \geq f(j|A\cup B) \implies g(j|A) \geq 0$ and hence we have monotonicity.  
We can also see that $I_f(A;B) = f(A) - f(A \cup B) + f(B)$ is a difference of submodular functions since both $f(A)$ and $f(A \cup B)$ are submodular functions in $A$ when $B$ is fixed. A similar argument for the other case when set $A$ is fixed and we view $I_f(A;B)$ as a function of $B$.

Next, we study the submodularity of $g(A)$. First, we show that non-negativity of the third order partial derivatives $f^{(3)}(i,j,k;A) = f^{(2)}(j,k;A \cup i) - f^{(2)}(j,k;A)$ is the same as monotonicity of the second-order partial derivatives. Note that 
\begin{align}
f^{(3)}(i,j,k;A) \geq 0 \implies f^{(2)}(j,k;A \cup i) \geq f^{(2)}(j,k;A),    
\end{align}
which by induction, means 
\begin{align}
f^{(2)}(j,k;A \cup B) \geq f^{(2)}(j,k;A), \forall A,B \subseteq B
\end{align}
Next, recall that $g(j | A) = f(j | A) - f(j | A \cup B)$. Now since $f$ additionally satisfies the  monotonicity property for the second-order partial derivatives, we have: \begin{align}
 & f^{(2)}(j,k;A) \leq f^{(2)}(j,k;A \cup B) \nonumber \\
 \implies & f(j | A \cup k) - f(j | A) \leq f(j | A \cup B \cup k) - f(j | A \cup B) \nonumber \\
 \implies & f(j | A \cup k) - f(j | A \cup B \cup k) \leq f(j | A) - f(j | A \cup B) \nonumber \\
\implies & g(j | A \cup k) \leq g(j | A), \forall j, k \notin A. 
\end{align}
This means that is $C \supseteq A, g(j | C) \leq g(j | A)$ which implies that $g$ is submodular. To prove the converse, assume that $I_f(A;B)$ is submodular, but $f^{(2)}(j,k;A)$ is not necessarily monotone. In other words, there exists sets $A, C$, with $A \cap C = \emptyset$ such that $f^{(2)}(j,k;A) \geq f^{(2)}(j,k;A \cup C)$. Define $g(A) = I_f(A;C)$. Following the chain of inequalities like above, this implies that $g(j | A \cup k) \geq g(j | A)$ which contradicts the submodularity of $I_f(A;C)$ for any fixed set $C$. 

Finally, we note that the submodularity of $I_f(A; B | C)$ for fixed sets $B, C$ holds from Lemma~\ref{lemma:app-submod-MI-CMI-rel}. In particular, define $g(A) = f(A | C)$ and observe that $I_g(A; B)$ is submodular in $A$ for given $B, C$ if $g^{(3)}(i,j,k;A) \geq 0$. Now since $f^{(3)}(i,j,k;A) \geq 0, \forall A \subseteq \Omega$, it holds that $f^{(3)}(i,j,k;A \cup C) \geq 0, \forall A \subseteq \Omega$ and hence $g^{(3)}(i,j,k;A) \geq 0$.
\end{proof}

\subsubsection{Examples of when is $I_f(A; B)$ submodular}
Next, we study subclasses of submodular functions which satisfy the condition $f^{(2)}(j,k;A) = f(A \cup j \cup k) - f(A \cup i) - f(A \cup j) + f(A)$ is monotonically increasing in $A$, or in other words, $I_f(A; B)$ is submodular in $A$ for a fixed $B$. 

We first show that the Facility Location satisfies this condition. We start with the $\max$-function.
\begin{lemma}
If $f(A) = \max_{i \in A} w_i$, then $f^{(2)}(j,k;A)$ is a monotonically increasing function in $A \subseteq \Omega \backslash \{j,k\}$ for a given $j,k$.
\end{lemma}
\begin{proof}
Define $W_A = \max_{i \in X} w_i$. Notice that $f^{(2)}(j,k;A) = \max(W_A,w_j,w_k) - \max(W_A,w_j) - \max(W_A,w_k) - w_A$. We try to analyze if this function is monotone in $A$. Lets assume first that $w_A$ is smaller than both $w_i,w_j$. In this case, $f^{(2)}(j,k;A)$ is either $w_A - w_j \leq 0$ or $w_A - w_k \leq 0$ depending on whether $w_j \leq w_k$ or $w_j \geq w_k$. Furthermore, if $w_A$ is larger than either $w_j$ or $w_k$, notice that $f^{(2)}(j,k;A) = 0$. Hence $f^{(2)}(j,k;A)$ is an increasing function and hence proved.
\end{proof}
A simple observation is that if $f_1$ and $f_2$ both satisfy the property of the non-negativity of the third-order partial derivatives, any convex combination of $f_1$ and $f_2$ will also satisfy this. Hence, this implies that the Facility Location satisfies this property as well.
\begin{corollary}
If $f(A)$ is the Facility Location Function, then $f^{(2)}(j,k;A)$ is a monotonically increasing function in $A \subseteq \Omega \backslash \{j,k\}$ for a given $j,k$.
\end{corollary}
Next, we show that the Set Cover and Concave over Modular functions with Power Functions satisfy this.
\begin{lemma}
If $f(A) = w(\gamma(A))$ or $f(A) = [w(A)]^a, a \in [0,1]$, then $f^{(2)}(j,k;A)$ is a monotonically increasing function in $A \subseteq \Omega \backslash \{j,k\}$ for a given $j,k$.
\end{lemma}
\begin{proof}
Lets start with the Set Cover Function. Note that the set-cover function satisfies $\gamma(A \cup B) = \gamma(A) \cup \gamma(B)$. Plugging this into the definition of the second-order partial derivatives, we get $f^{(2)}(j,k;A) = w(\gamma(A) \cup \gamma(i) \cup \gamma(j)) - w(\gamma(A) \cup \gamma(i)) - w(\gamma(A) \cup \gamma(j)) + w(\gamma(A)) = w(\gamma(A) \cap \gamma(i) \cap \gamma(j))$ which is monotonically increasing since the set cover is monotonically increasing.

Next, we look at concave over modular functions: $f(A) = [w(A)]^a$. Then, $f^{(2)}(j,k;A) = [w(A) + w(j) + w(k)]^a - [w(A) + w(j)]^a - [w(A) + w(k)]^a + [w(A)]^a$. To show that this is monotone, we look at its continious extension: $g(x) = [x + w(j) + w(k)]^a - [x + w(j)]^a - [x + w(k)]^a + [x]^a$ at $x = w(A)$. Note that $g^{\prime}(x) = a/[x + w(j) + w(k)]^{1-a} - a/[x + w(j)]^{1-a} - a/[x + w(k)]^{1-a} + a/x^{1-a}$. Finally, we can see that $g^{\prime}(w(A)) \geq 0$ since the function $h(A) = a/w(A)^{1-a}$ is a supermodular function in $A$.
\end{proof}

A corollary of the above is that sums of power concave over modular functions satisfy the monotone double gain property. We end this section, by showing, somewhat surprisingly that a simple uniform matroid rank function does not satisfy this property (and hence not all concave over modular functions satisfy this).

\begin{lemma} \label{lemma:app-non-submod-MI}
If $f(A) = \min(|A|, k)$, then $f^{(2)}(j,k;A)$ is not necessarily a monotonically increasing function in $A \subseteq \Omega \backslash \{j,k\}$ for a given $j,k$.
\end{lemma}
\begin{proof}
Lets start with $A = \emptyset$. Its easy to see that $f^{(2)}(j;k,A) = 0$. Next, set $A$ to be a set of size $k-1$. Now, $f^{(2)}(j;k,A) = k-k-k+(k-1) = -1$. Hence $f^{(2)}(j;k,A)$ is not monotonically increasing in $A$. Similarly, if $A$ is a set of size $k$, again, the value of $f^{(2)}(j;k,A) = 0$ and as a result, in this case $f^{(2)}(j;k,A)$ is neither increasing nor decreasing.
\end{proof}
Finally, we show that the Submodular Mutual Information is not submodular when $f$ is the Uniform Matroid Rank Function.
\begin{corollary} \label{corr:app-non-submod-MI}
If $f(A) = \min(|A|, k)$, then $I_f(A;B)$ neither submodular nor supermodular in $A$ for a fixed $B$.
\end{corollary}
Since $f^{(2)}(j;k,A)$ is neither increasing nor decreasing, $I_f(A; B)$ is neither submodular nor supermodular when $f$ is the Matroid Rank Function.

\subsubsection{Proof of Lemma \ref{lemma:submod-mutinfo-modularbounds-prop} (Modular Upper/Lower bounds of $I_f(A: B)$)}
\begin{proof} \label{proof:lemma-submod-mutinfo-modularbounds-prop}
For the modular lower bound we need to show that $f(A|B) \leq \sum_{j \in A \backslash B} f(j | B)$. Let the set $A \setminus B$ contain $k$ elements: $\{a_1, \dots, a_k\}$ and then we construct a chain of sets $X_0, X_1, X_2, ..., X_k$ s.t $X_i = X_{i-1} \cup \{a_i\}$ with $X_0 = B$ (and therefore $X_k = A \cup B)$. Note that 
\begin{align}
f(A|B) &= f(A\cup B) - f(B) \nonumber \\
& = f(X_k) - f(X_0) \nonumber \\
& = \sum_{i=1}^{k} f(X_i) - f(X_{i-1}) \nonumber \\
&= \sum_{i=1}^{k} f(X_{i-1} \cup \{a_i\}) - f(X_{i-1}) \nonumber \\
&= \sum_{i=1}^{k} f(\{a_i\} | X_{i-1})    
\end{align}

Also note that in the above construction we have $B \subseteq X_i \; \forall i=0,1, \dots ,k$
Hence by submodularity of $f$ we will have 
\begin{align}
& f(\{a_i\} | X_{i-1}) \leq f(\{a_i\} | B) \; \forall i \nonumber \\
\implies & f(A|B) = \sum_{i=1}^{k} f(\{a_i\} | X_{i-1}) \leq \sum_{i=1}^{k} f(\{a_i\} | B) = \sum_{j \in A \setminus B} f(j | B)    
\end{align}

For the modular upper bound, we need to show that $f(A|B) \geq \sum_{j \in A \setminus B} f(j | \Omega \setminus j)$. Using the construction we created in the previous case, we make a note that $X_{i-1} \subseteq \Omega \setminus \{a_i\}$ since $\{a_i\}$ does not lie in either of the sets by definition. Thus by the submodularity of $f$ we will have $f(\{a_i\} | X_{i-1}) \geq f(\{a_i\} | \Omega \setminus \{a_i\}) \; \forall i \implies f(A|B) = \sum_{i=1}^{k} f(\{a_i\} | X_{i-1}) \geq \sum_{i=1}^{k} f(\{a_i\} |\Omega \setminus \{a_i\}) = \sum_{j \in A\setminus B} f(j |\Omega \setminus j)$ 
\end{proof}
We note we can similarly achive upper and lower bounds for the submodular conditional mutual information.

\subsection{Multiset Submodular Mutual Information and MultiSet Submodular Total Correlation: Proofs}\label{sec:app-properties-multiset-tc-smi}

\subsubsection{Proof of Lemma \ref{lemma:pos-monotone-tc-prop} (Monotonicity and Positivity of Multi-Set Total Correlation)}
\begin{proof}
The proof follows from the assumption that $f$ is a monotone submodular function. This means that $f$ is also sub-additive and hence $f(A \cup B) \leq f(A) + f(B)$ for all $A,B \subseteq \Omega$. We can inductively prove this as: 
$ \sum_{i=1}^{k} f(A_i) \geq \sum_{i=1}^{k-2} f(A_i) + f(A_{k-1} \cup A_k) \geq \sum_{i=1}^{k-3} f(A_i) + f(A_{k-2} \cup A_{k-1} \cup A_k) \cdots \geq f(A_1) + f(\cup_{i = 2}^k A_i) \geq f(\cup_{i=i}^k A_i)$ which implies that $C_f(A_1, \dots, A_k) \geq 0$. 

We have monotonicity in any one argument from submodularity of $f$ since $C_f(A_1\cup j, A_2, \dots A_k) - C_f(A_1,\dots, A_k) = f(j|A_1) - f(j|\cup_{i=1}^{k} A_i) \geq 0$ since $A_1 \subseteq \cup_{i=1}^{k} A_i$ and hence the gain for just $A_1$ will always be larger giving us the required monotonicity. Finally, monotonicity in all arguments follows as a consequence of monotonicity in each argument through a simple inductive argument.
\end{proof}

\subsubsection{Proof of Lemma \ref{lemma:positivity-mi-prop} and Negativity of the $k$-Set Submodular Mutual Information}
\begin{proof}
We can re-write the expression for $I_f(A;B;C)$ in terms of the two-way mutual information as follows:$I_f(A;B;C) = I_f(A;B) + I_f(C;B) - I_f(A \cup C; B)$. Thus if $I_f(A;B)$ is submodular for a fixed $B$ we will have non-negativity and similarly if it is super modular, the three-way mutual information will be non-positive.
\end{proof}
Next, we provide a simple example of a submodular function where the three-way submodular mutual information is negative. 
\begin{lemma}
There exists a submodular function $f$ such that $I_f(A;B;C) < 0$.
\end{lemma}
\begin{proof}
This proof is very similar (and also is connected) to the non-submodularity of the two way mutual information from Lemma~\ref{lemma:app-non-submod-MI}. Similar to the construction in Lemma~\ref{lemma:app-non-submod-MI}, define $f(A) = \min(|A|,k)$ with $k = 10$. Let $A,B,C$ be three disjoint sets with $|A| = \frac{k}{2}-1$, $|B| = |C| = \frac{k}{2}$. Then $I_f(A;B;C) = f(A) + f(B) + f(C) - f(A \cup B) - f(A \cup C) - f(B \cup C) + f(A \cup B \cup C) = 3\frac{k}{2}-1 - 3k + 2 + k = 1 - \frac{k}{2} = -4$.
\end{proof}
Since the the 3-set submodular mutual information is negative, the $k$-set submodular mutual information is not necessarily non-negative.

\subsubsection{Proof of Lemma \ref{lemma:3way-mutinfo-monotone} and Monotonicity of the $k$-Set Submodular Mutual Information}
\begin{proof}
We begin by writing out the difference $I_f(A\cup D;B\cup D;C\cup D) - I_f(A;B;C)$ and show that it is non-negative using submodularity and monotonicity of $f$. We have $I_f(A\cup D;B\cup D;C\cup D) - I_f(A;B;C) = f(A \cup D) + f(B \cup D) + f(C \cup D) - f(A \cup B \cup D) - f(B \cup C \cup D) - f(A \cup C \cup D) + f(A \cup B \cup C \cup D) - f(A) - f(B) - f(C) + f(A \cup B) + f(B \cup C) + f(A \cup C) - f(A \cup B \cup C)$. On rearranging and grouping certain terms note that the difference is: $(f(A \cup D) + f(A \cup B) - f(A) - f(A \cup B \cup D)) + (f(B \cup D) + f(B \cup C) - f(B) - f(B \cup C \cup D)) + (f(C \cup D) + f(A \cup C) - f(C) - f(A \cup C \cup D)) + (f(A \cup B \cup C \cup D) - f(A \cup B \cup C))$. Now observe that each of the four terms is non-negative, the first three due to submodularity and the last one by monotonicity. 

For the case when $B,C$ are fixed, we do not necessarily have monotonicity in $I_f(A;B;C)$ as a function of $A$. This can be shown with a counter example where the function $f$ is a uniform matroid function, $f(A) = \min(|A|, k)$ and the sets $A, B, C$ are pairwise disjoint and such that $|A| = \frac{k}{2}-1$, $|B| = |C| = \frac{k}{2}$. In this case, notice that $I_f(A;B;C) = 1-\frac{k}{2}$ and $I_f(A \cup j;B;C) = -\frac{k}{2}$ and in fact $I_f(A \cup j;B;C) \leq I_f(A;B;C)$.
\end{proof}
Next, we show that the four-way multi-set submodular mutual information is not necessarily monotone even in all its arguments.
\begin{lemma} \label{lemma:app-4-way-smi-monotone}
The Given sets $A,B,C,D,E$, the four way Mutual Information does not necessarily satisfy $I_f(A \cup E;B \cup E;C \cup E;D \cup E) \geq I_f(A;B;C;D)$.
\end{lemma}
\begin{proof}
Again, to prove this, we rely on the Matroid rank function $f(A) = \min(|A|, k)$. In this case, let $A,B,C,D$ be disjoint sets satisfying $|A| = |B| = |C| = |D| = \frac{k}{2}-1$. Now note that $I_f(A;B)$ has 4 terms with singletons $f(A)$ (positive sign), 6 terms with pairs $f(A \cup B)$ (negative sign), 4 terms with triples $f(A \cap B \cap C)$ (positive sign) and one term $f(A \cup B \cup C \cup D)$ (negative sign). Hence it is easy to see that $I_f(A;B;C;D) = 2k-4 - 6k + 12 + 4k - k = 8 - k$. Next, consider (for $j \notin A,B,C,D$), $I_f(A \cup j, B \cup j, C \cup j, D \cup j)$. Again, notice that since all sets share the elements $j$, the singleton values are $k/2$, the values of the pairs is $k-1$, triples is $k$ and $f(A \cup B \cup C \cup D) = k$. Hence  $I_f(A \cup j, B \cup j, C \cup j, D \cup j) = 2k - 6k + 6 + 4k - k = 6 - k < 8 - k = I_f(A;B;C;D)$. Hence $I_f(A;B;C;D)$ is not monotone even in all its arguments.
\end{proof}
This proves that unlike the total correlation, $I_f(A_1;A_2;\cdots;A_k)$ is not necessarily monotone in any or all its variables for $k \geq 4$. 

\subsubsection{Proof of Lemma \ref{lemma:upperbound-mutinfo-prop} and Upper Bounds for the $k$-Set Submodular Mutual Information}
\begin{proof}
We first prove the result for $k = 3$. Begin with the definition of $I_f(A;B;C) = f(A) + f(B) + f(C) - f(A \cup B) - f(B \cup C) - f(A \cup C) + f(A \cup B \cup C)$. We will show that $I_f(A;B;C) \leq f(A)$ and using similar arguments we will also have $I_f(A;B;C) \leq f(B)$ and $I_f(A;B;C) \leq f(C)$. We begin with separating out $f(A)$ from other terms in the definition of $I_f(A;B;C)$ as follows: $I_f(A;B;C) = f(A) - (-f(B) - f(C) + f(A \cup B) + f(B \cup C) + f(A \cup C) - f(A \cup B \cup C)) = f(A) - (f(A|C) + f(C|B) - f(C|A\cup B))$. Now note that $f(A|C) + f(C|B) - f(C|A\cup B) \geq 0$ since the conditional gain $f(A|C)$ is always non-negative and also $f(C|B) - f(C|A\cup B)\geq 0$ by submodularity of $f$. Hence we then have $I_f(A;B;C) \leq f(A)$ and using similar arguments for $B,C$ we get that $I_f(A;B;C) \leq \min(f(A), f(B), f(C))$.

Next, we show that the upper bound also holds for the 4-set submodular mutual information. To prove the result for the 4-way case, we use the fact that the 3-way submodular mutual information is monotone in all its arguments, i.e. $I_f(A;B;C) \leq I_f(A \cup D;B \cup D;C \cup D)$. It is easy to see that $I_f(A;B;C;D) = f(D) + I_f(A;B;C) - I_f(A \cup D;B \cup D;C \cup D) \leq f(D)$. By symmetry, we get this for the other sets as well.

This same proof technique does not carry over to the $k$-set case, since in general, it is not necessary monotone in all its variables (cf. Lemma~\ref{lemma:app-4-way-smi-monotone}). Unfortunately, it does not hold for $k = 5$ and hence is not guaranteed to hold for $k \geq 5$. We prove this by showing an example. Define $f(A) = \min(|A|, 3k)$ and let $A,B,C,D,E$ be disjoint sets with $|A| = |B| = |C| = |D| = k, |E| = 1$. Note that the value of the RHS is $1$. Now, in the expansion of $I_f(A;B;C;D;E)$, note that there $5$ terms involving singletons: $f(A)$, $10$ terms like $f(A \cup B)$ (pairs of sets), $10$ terms like $f(A \cup B \cup C)$ (triplets of sets), $5$ terms like $f(A \cup B \cup C \cup D)$ and one term $f(A \cup B \cup C \cup D \cup E)$. Note that by definition, the terms $f(A \cup B \cup C \cup D) = 3k$ and similarly, $f(A \cup B \cup C \cup D \cup E) = 3k$. Moreover, in the expansion of $I_f$, the singletons have a positive sign, the pairs negative sign, triplets again positive, quadruplets have negative and finally $f(A \cup B \cup C \cup D \cup E)$ is positive. Moreover, the contribution of the singletons is $4k + 1$, the pairs is $16k + 4$, triplets is $24k + 6$ (since they are all not saturated). Furthermore, the contribution of the quadruplets is $15k$ (since it is saturated) and the last with all five sets is $3k$. This means, that $I_f(A;B;C;D;E) = 4k+1 - 16k - 4 + 24k + 6 - 15k + 3k = 3 > \min(f(A),f(B),f(C),f(D),f(E)) = 1$.
\end{proof}

Note that since $I_f(A_1;A_2; \dots, A_{k+1}) = I_f(A_1; \dots; A_k) - I_f(A_1;\dots,A_k|A_{k+1})$, we do have an upper bound whenever the submodular conditional multi-set information is non-negative since, w.l.o.g.\  assuming $f(A_1) \leq f(A_2) \leq \dots \leq f(A_{k+1})$, we have by induction that $I_f(A_1; \dots, A_k) \leq f(A_1) - I_f(A_1; \dots; A_k | A_{k+1}) \leq \min_i f(A_i)$. For example, it does hold for special cases of facility location and set cover.

\subsection{Submodular Variation of Information}
\subsubsection{Proof of Lemma \ref{lemma:psuedometric-def} (Variation of Information is a PsuedoMetric)}
\begin{proof}
We have non-negativity of $D_f(A,B)$ by the monotonic nature of $f$. $f(A \cup B) \geq f(A)$ and $f(A \cup B) \geq f(B) \implies D_f(A,B) = 2 f(A \cup B) - f(A) + f(B) \geq 0$ and we also have symmetry for $D_f(A,B)$ by definition. 
Next we show that triangle inequality also holds i.e: $D_f(A, C) \leq D_f(A, B) + D_f(B, C)$.

$$ f(A\cup B) + f(B \cup C) \geq f(A \cup B \cup C) + f((A \cup B) \cap (B \cap C)) \; \text{by submodularity}$$
$$ = f(B \cup (A \cup C)) + f(B \cup (A \cap C)) \; \text{by distributive properties}$$
$$ \implies f(A\cup B) + f(B \cup C) \geq f(A \cup C) + f(B) \; \text{by monotonicity}$$
$$\implies 2f(A\cup B) + 2f(B \cup C) \geq 2f(A \cup C) + 2f(B)$$
$$\implies 2f(A\cup B) + 2f(B \cup C) - 2f(B) - f(A) - f(C) \geq 2f(A \cup C) -f(A) - f(C) $$
$$\implies D_f(A, B) + D_f(B, C) \geq D_f(A, C)$$
This metric only fails to satisfy the identity of indiscernibles i.e $D_f(A,B) = 0 \notiff A=B$.
To show this, let $f(A) = \min(|A|, k)$ and let $A$ and $B$ be two disjoint sets of size $k$. Note that $D_f(A,B) = 2f(A \cup B) - f(A) - f(B) = 2k - k - k = 0$ even though $A$ and $B$ are disjoint! However we still have $D_f(A,B) = 0$ if $A=B$ and hence it is a pseudo-metric. 

Now the second part is if $\kappa_f > 0$, then $D_f(A,B) = 0 \implies A = B$. Let us assume that $A \neq B$. Recall that $D_f(A;B) = f(A|B) + f(B|A)$. We will show that when $A \neq B$, $f(A|B) > 0$ (with a similar argument for $f(B|A) > 0$) when $\kappa > 0$ giving us an contradiction that $D_f(A;B) > 0$. When $\kappa > 0$ we have $1 - \kappa = \min_{j \in \Omega} \frac{f(j | \Omega \setminus j)}{f(j | \emptyset)} < 1$.  
We use the following lower bound for $f(A|B)$ as follows: $f(A|B) \geq \sum_{j \in A \setminus B} f(j | \Omega \setminus j)$. Also from $\kappa > 0$ we have, $\sum_{j \in A \setminus B} f(j | \Omega \setminus j) > (1-\kappa) \sum_{j \in A \setminus B} f(j|\emptyset)$ and $f(j | \emptyset) > 0$ w.l.o.g. when $f$ is monotone\footnote{we can assume w.l.o.g since if $f(j | \emptyset) = 0, f(j | X) = 0, \forall X \subseteq V$ by submodularity. Hence this $j$ is a dummy element which can without loss of generality be removed from the ground set.}. Therefore $\sum_{j \in A \setminus B} f(j | \Omega \setminus j) > 0 \implies f(A|B) > 0$. Using a similar argument to show $f(B|A) > 0$, which implies that $D_f(A;B) > 0$ when $A \neq B$. 

\end{proof}

\subsubsection{Proof of Lemma \ref{lemma:psuedometric-bound} (Upper and Lower bounds with the Submodular Hamming Metric)}
\begin{proof}
For the upper bounds we have:$D_f(A,B) =  f(A \cup B) - f(A) + f(A \cup B) - f(B)$ and also $f(A \cup B) \leq f(B \setminus A) + f(A)$ and $f(A \cup B) \leq f(A \setminus B) + f(B) \implies D_f(A,B) \leq f(B \setminus A) + f(A \setminus B) = D^{SHA}(A,B)$. 
Now we also show a tighter upper bound by noting that $f(A\cup B) \leq f(A \cap B) + f(A \Delta B)$ (by submodularity of $f$) and that $I_f(A,B) \geq f(A \cap B)$. Hence, $D_f(A,B) = f(A \cup B) - I_f(A,B) \leq f(A \cup B) - f(A \cap B) \implies D_f(A,B) \leq f(A \Delta B)$.

Next, we show the lower bound. note that $D_f(A, B) = f(A | B) + f(B | A)$. Next, we show that $f(A | B) = f(A \cup B) - f(B) \geq (1 - \kappa(A \cup B)) f(A \backslash B)$. To show this, note that $f(A | B) \geq \sum_{j \in A \backslash B} f(j | A \cup B \backslash j) \geq (1 - \kappa(A \cup B)) \sum_{j \in A \backslash B} f(j | \emptyset) \geq (1 - \kappa(A \cup B)) f(A \backslash B)$. We can similarly show that $f(B | A) = \geq (1 - \kappa(A \cup B)) f(B \backslash A)$. Hence we have that $D_f(A, B) \geq (1 - \kappa(A \cup B)) [f(A \backslash B) + f(B \backslash A)] \geq (1 - \kappa(A \cup B)) f(A \Delta B)$.
\end{proof}


\section{Proofs of the Results in Section~\ref{sec:examples}}
\label{sec:app-proofs-sec4}
In this section, we prove the results from Section~\ref{sec:examples}.

\subsection{Proof of Lemma~\ref{lemma:def-modular} (Modular Function)}
\begin{proof}
Here the function $f$ is modular i.e $f(A) = w(A) = \sum_{a \in A} w(a)$. Now, $f(A\cup B) = \sum_{i \in A \cup B} w(i) = \sum_{i \in A} w(i) + \sum_{i \in B} w(i) - \sum_{i \in A \cap B} w(i) = w(A) + w(B) - w(A\cap B) = f(A) + f(B) - f(A \cap B)$. Hence $I_f(A;B) = w(A \cap B)$. For the conditional gain, we have $f(A|B) = f(A \cup B) - f(B) = w(A) - w(A\cap B) = w(A \setminus B)$ and $D_f(A,B) = f(A|B) + f(B|A) = w(A \setminus B) + w(B \setminus A) = w(A \Delta B)$ since $(A \setminus B) \cap (B \setminus A) = \emptyset$. 

Next, we study the multi-set mutual information. Recall that the expression is a simple inclusion-exclusion principle expression and when $f$ is modular, the term $I_f(A_1; \cdots; A_k) = -\sum_{T \subseteq [k]} (-1)^{|T|} f(\cup_{i \in T} A_i) = -\sum_{T \subseteq [k]} (-1)^{|T|} \sum_{j \in \cup_{i \in T} A_i} w_j = w(\cap_{i = 1}^k A_i)$.
\end{proof}

\subsection{Proof of Lemma~\ref{lemma:setcover} (Set Cover Function)}
\begin{proof}
Here $f$ is a set cover function, $f(A) = w(\cup_{a \in A} \gamma(a))$. Now, $f(A\cup B) = w(\cup_{c \in A\ \cup B} \gamma(c)) $ and we also have $\gamma(A \cup B) = \gamma(A) \cup \gamma(B)$ where $\gamma(A) = \cup_{a \in A} \gamma(A)$. We have the following result by the inclusion exclusion principle.
\begin{align}
I_f(A;B) &= f(A) + f(B) - f(A \cup B) \nonumber \\
 & = w(\gamma(A)) + w(\gamma(B)) - w(\gamma(A) \cup \gamma(B)) \nonumber \\
 &= w(\gamma(A) \cap \gamma(B))  
\end{align}
For the Conditional Gain, we have:
\begin{align}
 f(A|B) &= f(A \cup B) - f(B) \nonumber \\
  & = w(\gamma(A) \cup \gamma(B)) - w(\gamma(B)) 
  & = w(\gamma(A) \setminus \gamma(B))   
\end{align}
The pseudo metric also follows in a similar fashion, \begin{align}
D_f(A,B) &= f(A \cup B) - I_f(A;B) \nonumber \\
    &= w(\gamma(A) \cup \gamma(B)) - w(\gamma(A) \cap \gamma(B)) \nonumber \\
    &= w(\gamma(A) \setminus \gamma(B)) + w(\gamma(B) \setminus \gamma(A))    
\end{align}
Next, consider the multi-set mutual information. Recall that $I_f(A_1; \cdots; A_k) = -\sum_{T \subseteq [k]} (-1)^{|T|} f(\cup_{i \in T} A_i)$. When $f(A) = w(\gamma(A))$ is the set cover function, observe that $\gamma(\cup_{i \in T} A_i) = \cup_{i \in T} \gamma(A_i)$ and hence
\begin{align}
I_f(A_1; \cdots; A_k) = -\sum_{T \subseteq [k]} (-1)^{|T|} w(\cup_{i \in T} \gamma(A_i))
\end{align}
From the expression of the modular function (or essentially, the inclusion-exclusion property), observe that $I_f(A_1; \cdots; A_k) = w(\cap_{i = 1}^k \gamma(A_i))$.
\end{proof}

\subsection{Proof of Lemma~\ref{lemma:def-probsetcover} (Probabilistic Set Cover Function)}
\begin{proof}
Here $f(A)$ is the probabilistic set cover function: $f(A) = \sum_{i \in U} w_i (1 - \Pi_{a \in A} (1-p_{ia}))$. Here $p_{ia}$ is the probability that the element $a \in A$ covers the concept $i$ and $U$ is the set of all concepts. We use $P_i(A) = \Pi_{a \in A} (1-p_{ia})$ which denotes the probability that none of the elements in $A$ cover the concept $i$. Therefore $1 - P_i(A)$ will denote that at least one element in $A$ covers $i$. $I_f(A;B) = f(A) + f(B) - f(A \cup B) = 
 =  \sum_{i \in U} w_i ( 1 - P_i(A) + 1-P_i(B) - 1 - P_i(A\cup B)) 
 =  \sum_{i \in U} w_i (1 - (P_i(A) + P_i(B) - P_i(A\cup B)))$. 
 With disjoint $A,B$ we will have: $ P_i(A\cup B)) = P_i(A) P_i(B)$ and hence:
$I_f(A;B) =  \sum_{i \in U} w_i (1 - (P_i(A) + P_i(B) - P_i(A) P_i(B))) 
=  \sum_{i \in U} w_i (1 - P_i(A)) (1 - P_i(B)) $. 

For the conditional gain, we have $ f(A|B) = f(A \cup B) - f(B)  
= \sum_{i \in U} w_i [P_i(B) - P_i(A\cup B)] 
= \sum_{i \in U} w_i [\Pi_{b \in B} (1-p_{ib}) - \Pi_{c \in A \cup B} (1-p_{ic})]
= \sum_{i \in U} w_i [\Pi_{b \in B} (1-p_{ib}) - \Pi_{b \in B} (1-p_{ib}) \Pi_{a' \in A \setminus B} (1-p_{ia'})]
= \sum_{i \in U} w_i [\Pi_{b \in B} (1-p_{ib}) (1 - \Pi_{a' \in A \setminus B} (1-p_{ia'}))]
= \sum_{i \in U} w_i \; P_i(B) \; (1 - P_i(A \setminus B)) $
Use a similar strategy for the pseudo metric $ D_f(A,B) = f(A \cup B) - I_f(A;B) = 2 f(A \cup B) - f(A) - f(B) = f(A \cup B) - f(A) + f(A \cup B) - f(B) $ we will have $D_f(A,B) = \sum_{i \in U} w_i [P_i(B)(1 - P_i(A \setminus B)) + P_i(A) (1 - P_i(B \setminus A))]$.

It is easy to see that when $A$ and $B$ are disjoint, $I_f(A;B) =  \sum_{i \in U} w_i (1 - P_i(A)) (1 - P_i(B))$. As a result, note that $A \perp_f B$ iff $\forall i \in U$, $P_i(A)$ and $P_i(B)$ are not both $0$ (i.e. the probability that concept $i$ is covered by both sets $A, B$ is not $1$). Next, we provide the expression for conditional mutual information with the probabilistic set cover. Note that $I_f(A; B | C) = I_f(A; B) - I_f(A;B;C)$ and hence with the prob. set cover function, $I_f(A; B | C) = \sum_{i \in U} w_i (1 - P_i(A)) (1 - P_i(B))P_i(C)$. 
\end{proof}

\subsection{Proof of Lemma \ref{lemma:facloc} (Facility Location)}
\begin{proof}
Here we have the facility location set function, $f(A) = \sum_{i \in [n]} \max_{a \in A} s(i,a)$ where $s$ is similarity kernel and $\Omega= [n]$. 
\begin{align}
   I_f(A;B) &= f(A) + f(B) - f(A \cup B) 
= \sum_{i \in [n]} \max_{a \in A} s(i,a) + \max_{b \in B} s(i,b) - \max_{c \in A \cup B} s(i,c) \\
&= \sum_{i \in [n]} \max_{a \in A} s(i,a) + \max_{b \in B} s(i,b) - \max(\max_{a \in A} s(i,a), \max_{b \in B} s(i,b)) \\
&= \sum_{i \in \Omega} \min(\max_{a \in A} s(i,a), \max_{b \in B} s(i,b))
\end{align}
Assuming $s(i,i) = 1$ is the maximum similarity score in the kernel, we can then break down the sum over elements in ground set $\Omega$ as follows. For any $i \in A, \max_{a \in A} s(i,a) = 1$ and hence the minimum (over sets $A$ and $B$) will just be the term corresponding to B (and a similar argument follows for terms in $B$). $ I_f(A;B) = \sum_{i \in \Omega\setminus (A \cup B)} \min(\max_{a \in A} s(i,a), \max_{b \in B} s(i,b)) 
    + \sum_{i \in A \setminus B} \max_{b \in B} s(i,b) 
    + \sum_{i \in B \setminus A} \max_{a \in a} s(i,a)
    + \sum_{i \in A \cap B} 1 $.
A special case arises when $B = \Omega \setminus A$ as the first and last sums disappear thereby making it look  like a symmetric version of the facility location function: 
$$ I_f(A ; \Omega\setminus A) = \sum_{i \in A} \max_{b \in \Omega \setminus A} s(i,b) + \sum_{i \in \Omega \setminus A} \max_{a \in A} s(i,a).$$

\noindent For the Conditional Gain we have $f(A|B) = 
 = \sum_{i \in \Omega} \max(\max_{a \in A} s(i,a), \max_{b \in B} s(i,b)) -  \max_{b \in B} s(i,b)  
 = \sum_{i \in \Omega} \max(0, \max_{a \in A} s(i,a) - \max_{b \in B} s(i,b))$.
Thus for an $i \in \Omega$, the notion of "gain" makes sense wrt sets $A, B$ if $A$ has a higher similarity element than the highest one from $B$. Similarly for the pseudo metric we get:
\begin{align}
    D_f(A,B) &= f(A \cup B) - I_f(A;B) \\
&= \sum_{i \in \Omega} \max(\max_{a \in A} s(i,a), \max_{b \in B} s(i,b)) - \min(\max_{a \in A} s(i,a), \max_{b \in B} s(i,b)) \\
&= \sum_{i \in \Omega} | \max_{a \in A} s(i,a) - \max_{b \in B} s(i,b)|
\end{align}
Next, lets get to the multi-set mutual information. Recall, that the multi-set mutual information is defined as: $I_f(A_1; \cdots; A_k) = -\sum_{T \subseteq [k]} (-1)^{|T|} f(\cup_{i \in T} A_i)$. We then  obtain the following relationship:
\begin{align}
    I_f(A_1; \cdots; A_{k+1}) = I_f(A_1; \cdots; A_{k}) - I_f(A_1 \cup A_{k+1}; \cdots; A_k \cup A_{k+1}) + f(A_{k+1})
\end{align}
We then prove the multi-set relation for the $\max$ function, and since the Facility Location is a sum of $\max$ functions, the result will extend there as well. Define $M(X) = \max_{i \in X} w_i$. Let us assume, by induction, that the result holds for $k$. In other words, let $I_f(A_1; \cdots; A_k) = \min_{i = 1:k} M(A_i)$. Using the above inductive relationship, we obtain:
\begin{align}
    I_f(A_1; \cdots; A_{k+1}) = \min_{i = 1:k} M(A_i) - \min_{i = 1:k} \max\{M(A_i), M(A_{k+1})\} + M(A_{k+1})
\end{align}
Now, assume that $M(A_{k+1}) \leq  \min_{i = 1:k} M(A_i)$. In this case, observe that $ I_f(A_1; \cdots; A_{k+1}) = \min_{i = 1:k} M(A_i) - \min_{i = 1:k} M(A_i) + M(A_{k+1}) = M(A_{k+1})$. Hence, $I_f(A_1; \cdots; A_{k+1}) = \min_{i = 1:k+1} M(A_i)$. For the second case, we assume that $M(A_{k+1}) >  \min_{i = 1:k} M(A_i)$. In this case, it is easy to see that $\min_{i = 1:k} \max\{M(A_i), M(A_{k+1})\} = M(A_{k+1})$ since for all $i$'s such that $M(A_i) < M(A_{k+1}), \max\{M(A_i), M(A_{k+1})\} = M(A_{k+1})$. Hence, in this case, $I_f(A_1; \cdots; A_{k+1}) = \min_{i = 1:k} M(A_i) = \min_{i = 1:k+1} M(A_i)$. Hence in both cases, the result is the minimum among the $M(A_i)$ values. Since the Facility Location is the sum of the respective Max functions, the multi-set mutual information is the sum of the minimums.
\end{proof}






\subsection{Proof of Lemma~\ref{lemma:def-graphcut} (Generalized Graph Cut)}
\begin{proof}
Here we have the generalized graph cut set function, $f(A) = \lambda \sum_{i \in \Omega} \sum_{a \in A} s_{i a} - \sum_{a_1,a_2 \in A} s_{a_1 a_2} $ with $\lambda \geq 2$. $I_f(A;B) = f(A) + f(B) - f(A \cup B) 
= \lambda \sum_{i \in \Omega} \sum_{a \in A} s_{i a} - \sum_{a_1,a_2 \in A} s_{a_1 a_2} 
+ \lambda \sum_{i \in \Omega} \sum_{b\in B} s_{i b} - \sum_{b_1,b_2 \in B} s_{b_1 b_2}
- \lambda \sum_{i \in \Omega} \sum_{c \in A \cup B} s_{i c} + \sum_{c_1,c_2 \in A \cup B} s_{c_1 c_2}
$. A special case arises with disjoint $A,B$, since $\sum_{c \in A \cup B}$ can be broken down as 
$\sum_{c \in A} + \sum_{c \in B}$ and hence $f(A) + f(B)$ part gets eliminated leaving
behind a sort of "cross-similarity" between $A,B$ as follows.
$I_f(A;B) =  2 \sum_{a \in A} \sum_{b \in B} s_{a b}$ and hence 
$I_f(A;\Omega \setminus A) =  2 \sum_{a \in A} \sum_{a \in \Omega \setminus A} s_{a a'} $.

For the Conditional Gain we break down $\sum_{c \in A \cup B}$ as $\sum_{c \in A - B} + \sum_{c \in B}$ and 
hence obtain $f(A|B) = f(A \cup B) - f(B)$ as follows: Firstly, $f(A \cup B) = 
\lambda \sum_{i \in \Omega} \sum_{c \in A \setminus B} s_{i a} 
+ \lambda \sum_{i \in \Omega} \sum_{b\in B} s_{i b} 
- \sum_{b_1,b_2 \in B} s_{b_1 b_2} 
- \sum_{c_1,c_2 \in A \setminus B} s_{c_1 c_2}
-  2 \sum_{a' \in A \setminus B} \sum_{b \in B} s_{a' b}
$. This implies that $ f(A|B) = f(A \setminus B) -  2 \sum_{a' \in A \setminus B} \sum_{b \in B} s_{a' b}$. For the special case of disjoint $A,B$ we then obtain: $f(A|B) = f(A) -  2 \sum_{a' \in A} \sum_{b \in B} s_{a b} \; \text{(for disjoint A,B)}$. Finally, note that $A \perp_f B$ iff $I_f(A; B) = 0$ iff $s_{a b} = 0, \forall a \in A, b \in B$.
\end{proof}




\section{Proofs relating to Optimization Problems in Section~\ref{sec:optimization-problems}}
\label{sec:app-proofs-sec5}

\subsection{Proof of Lemma \ref{lemma:complement-mutinfo-approx-monotone} (Submodular Mutual Information Based Selection)}
Below is the proof of Lemma \ref{lemma:complement-mutinfo-approx-monotone} that $I_f(A; \Omega \setminus A)$ is approximately monotone.
\begin{proof}
We have $g(A) = I_f(A ; \Omega \setminus A) = f(A) + f(\Omega \setminus A) - f(\Omega)$. In order to investigate the monotonicity of $g(A)$, we consider the gain of adding an element $j \notin A$ to $A$, $g(j|A) = g(A \cup \{j\}) - g(A) = f(A \cup \{j\}) - f(A) + f(\Omega \setminus (A \cup \{j\})) - f(\Omega \setminus A)$. Thus we can write the gain as $g(j|A) = f(j|A) - f(j|\Omega \setminus (A \cup \{j\}))$. Now this difference can be both positive and negative in general and hence we require the notion of approximate monotonicity. Define $\kappa_f(A) = \max_{j \in \Omega \setminus A} \frac{f(j|\Omega \setminus (A \cup \{j\}))}{f(j)}$ as the curvature for a set $A$. Then we have that $g(j|A) \geq f(j|A) - \kappa_f(A) f(j)$ and since $f(j|A) \geq 0$ ($f$ is still a monotone function), we will have that $g(j|A) \geq -\kappa_f(A) f(j)$. If we make the assumption that $f(j) \leq 1 \; \forall j \in \Omega$, then we will have $g(j|A) \geq -\kappa_f(A)$.

We will follow a similar proof strategy as seen in \citep{krauseNearOptimalSensorPlacementsa, nemhauser1978analysis} to show the guarantee. Let the greedy algorithm select the elements $a_1, a_2, \dots, a_k$ at every iteration and let $A_i = \{a_1, \dots, a_i\}$. As defined in the lemma, $A^*$ represents with optimal subset with elements $\{o_1, o_2, \dots, o_k\}$ (in any order) and $g(A) = I_f(A; \Omega \setminus A)$ is the submodular function we are maximizing. Now from the result of Lemma \ref{lemma:complement-mutinfo-approx-monotone}, we have: 
\[    g(A_i \cup A^*) \geq g(A^*) - k \kappa_f(A^*) \leq g(A_i \cup A^*) \]
\[ \leq g(A_i \cup A^*) = g(A_i) + \sum_{j=1}^{k} g(o_j| A_i \cup \{o_1, o_2, \dots,o_i\}) \]
\[ \leq g(A_i) + \sum_{j=1}^{k} g(o_j | A_i) \; \text{by submodularity since $A_i \subseteq A_i \cup \{o_1, o_2, \dots,o_i\}$}\]
\[ \leq g(A_i) + k g(a_{i+1} | A_i) \]
We have the last inequality as a consequence of the greedy algorithm procedure where by definition $a_{i+1} = \arg \max_{j \in \Omega} g(j | A_i)$. Note that $g(a_{i+1} | A_i) = g(A_{i+1}) - g(A_i)$ and then on some rearrangements we will have $(g(A^*) - k \kappa_f(A^*)) - g(A_{i+1}) \leq (1 - \frac{1}{k})(g(A^*) - k \kappa_f(A^*)) - g(A_{i}) \implies (g(A^*) - k \kappa_f(A^*)) - g(A_{k}) \leq (1 - \frac{1}{k})^k (g(A^*) - k \kappa_f(A^*)) $ since $A_0 = \emptyset$ with functional value of zero and $\hat{A} = A_k$. Using the bound $1-x \leq \exp^{-x}$ we will have the required bound $g(\hat{A}) \geq (1 - \frac{1}{e})(g(A^*) - k \kappa_f(A^*))$.
\end{proof}

\subsection{Query Based and Privacy Preserving Summarization}
In this subsection, we study the approximation guarantees and hardness results of the various formulations of query based and privacy preserving summarization.
\subsubsection{Query Based Summarization: Proof of Theorem~\ref{lemma-smimax-approxfactor}}
In this section, we study Problem SMIMax:
\begin{align}
    \max_{A \subseteq \Omega, |A| \leq k} I_f(A; Q) + \lambda g(A)
\end{align}
Here $f$ and $g$ are monotone and non-negative submodular functions and $\lambda \geq 0$. Below is the proof of Theorem~\ref{lemma-smimax-approxfactor}.
\begin{proof}
The proof of the first part is a direct corollary of Theorem~\ref{theorem:submod-mutinfo-mon-submod-prop}. In particular, if $f^{(3)}(i,j,k;A) \geq 0$, we have that $I_f(A; Q)$ is submodular in $A$ for a given $Q$. Also note that, $I_f(A; Q)$ is always monotone, and hence if $g$ is monotone submodular, it implies that $I_f(A; Q) + \lambda g(A)$ is monotone submodular and hence the $1 - 1/e$ approximation holds from~\citep{nemhauser1978analysis}.

The proof of the hardness follows the construction from \citep{svitkinaSubmodularApproximationSamplingbased2010, goemansApproximatingSubmodularFunctions2009} where we construct two submodular functions $f(A)$ and $f_R(A)$ which are indistinguishable from each other with high probability. Define $f$ as the uniform rank $\alpha$ matroid, $f(A) = min\{|A|, \lambda \}$. Furthermore $R$ be a subset of $\Omega$ with $|R| = \lambda$ and define $g$ as $f_R(A) = \min \{|A|, \beta + |S \cap R^c|, \lambda \}$. Assume that $\lambda \approx \sqrt{n}$ and $\beta = \Omega(\log n)$ and $R^c$ denotes the complement of $R$. It can be shown using the Chernoff bound analysis seen in \citep{svitkinaSubmodularApproximationSamplingbased2010, goemansApproximatingSubmodularFunctions2009} that $f$ and $f_R$ can only be distinguished from each other using a polynomial number of queries with probability no more than $n^{\omega(1)}$.

Now consider the problem instance under the context of $f$ as : $\max_{A \subseteq \Omega, |A| \leq \lambda-1} I_f(A;Q) + g(A)$ where $g(A) = \frac{|A|}{(\lambda -1)\alpha(n)}$. Also consider the similar problem under the context of $f_R$ as: $\max_{A \subseteq \Omega, |A| \leq \lambda-1} I_{f_R}(A;Q) + g(A)$. Here we also define the fixed set $Q$ to be a singleton set $\{q\}$ such that $q \in R$. 

Note that under the constraint $|A| \leq \lambda-1$, we will have $I_f(A;Q) = 0$ for any $A$ satisfying $|A| \leq \lambda - 1$. This holds because $f$ is actually modular in that range. Furthermore, the maximum value of of $g(A) = \frac{1}{\alpha(n)}$ when $|A| = \lambda -1$ and hence the maximum achievable value of $I_f(A;Q) + g(A)$ under the constraints is $\frac{1}{\alpha(n)}$. Moreover, any algorithm which tries to maximize the two problems (with $I_f$ and $I_{f_R}$ cannot distinguish between them in polynomial number of queries and hence will not be able to obtain a better solution than $\frac{1}{\alpha(n)}$. However noting the construction of $f_R$, we know maximum solution attainable is $1 + \frac{1}{\alpha(n)}$ when $A = R \setminus q$. In particular, note that $I_{f_R}(A; Q) = f_R(A) + f_R(Q) - f_R(A \cup Q) = f_R(R \setminus q) + f_R(q) - f_R(R) = f_R(q) = 1$. Furthermore, $g(R \setminus q) = \frac{1}{\alpha(n)}$ and hence the worst case approximation factor is at least $\frac{1 + \frac{1}{\alpha(n)}}{\alpha(n)} = 1 + \alpha(n) > \alpha(n)$. This means that any polynomial time algorithm which gives us a better factor should also be able to distinguish between $f$ and $g$ reliably giving us a contradiction.      
\end{proof}
Finally, we mention that no query constraint is similar to setting $Q = V$, in which case $I_f(A; Q) = f(A)$, which is the original objective for summarization in the absence of a query. In the next section, we consider an alternate formulation for query based summarization.

\subsubsection{Query Based Summarization: Formulation as a constrained problem}
Here we consider a problem that involves the use of the conditional gain function as a constraint in a submodular maximization problem. Recall that in the query based selection, we have a fixed set $Q \subseteq \Omega$, and the goal is to achieve a good \emph{summary} which is also related to the query set. Below is the constrained formulation of this problem:

\begin{align} \label{opt:scsk-query}
\max_{A \subseteq \Omega} g(A) \;\text{s.t}\; f(A|Q) \leq \epsilon, |A| \leq k   
\end{align}
Note that this problem is related to maximizing the mutual information $I_f(A; Q) + \lambda g(A)$ since $I_f(A; Q) = f(A) - f(A | Q)$. Hence maximizing $I_f(A; Q)$ is related to minimizing $f(A | Q)$. We can therefore decouple the objective and the constraint, and have $f(A | Q)$ which models the query similarity in the constraint. 

Since both $f(A|B)$ and $g(A)$ are submodular functions in $A$, the problem is an instance of submodular maximization with multiple submodular knapsack constraints which has bounded approximation guarantees of $\left[1 - 1/e, \frac{n}{ 1 + (n - 1)(1 - \kappa_f)}\right]$ from the result of \citep{iyerSubmodularOptimizationSubmodular}
where $\kappa_f = 1 - \min_{j \in V} \frac{f(j | V \backslash j)}{f(j)}$ is the curvature of the submodular function. An interesting observation about the constrained formulation is that it always admits a bounded approximation guarantee as long as $f$ and $g$ are monotone submodular. SMIMax on the other hand, requires additional assumptions for the problem to admit bounded approximation guarantees.

Next, we point out the similarity between SMIMax and the constrained formulation above. In particular, the parameter $\epsilon$ is similar to the trade-off parameter $\lambda$, except that via the constraint, we have explicit control over the query similarity, rather than a somewhat indirect effect via the tradeoff term $\lambda$. This is mainly because, thanks to the constraint, we are able to decouple the main objective (to have that only dependent on $A$) and the query term (which is only in the constraint). Finally, note that when $Q = V$, the constraint $f(A | Q) \leq \epsilon$ disappears (since $f(A | V) = 0$) and we get cardinality constrained submodular maximization.

\subsubsection{Privacy Preserving Summarization: Proof of Lemma~\ref{lemma-approx-nsmimax-cgmax} and comparing NSMIMax and CGMax}
In this section, we study Problems NSMIMax and CGMax. We start with Problem NSMIMax.
\begin{align}
    \max_{A \subseteq \Omega, |A| \leq k} \lambda g(A) - I_f(A; P) = \max_{A \subseteq \Omega, |A| \leq k} \lambda g(A) + f(P | A)
\end{align}
Our first result shows the conditions for achieving approximation factors for NSMIMax and its hardness.
\begin{lemma}\label{lemma-approx-nsmimax}
NSMIMax is an instance of (non-monotone) submodular maximization if the third-order partial derivatives are not positive, i.e. $f^{(3)}(i,j,k;A) = f^{(2)}(j,k;A \cup i) - f^{(2)}(j,k;A) \leq 0$. In the worst case, however, there exists submodular function $f$ such that NSMIMax is inapproximable. 
\end{lemma}
\begin{proof}
First observe that NSMIMax is non-monotone because $g(A)$ is a monotone submodular function and $I_f(A; P)$ is also monotone, and hence NSMIMax is a difference of monotone functions. Since, $g$ is monotone submodular, NSMIMax can potentially be approximable, if $I_f(A; P)$ is supermodular in $A$ for fixed $P$. Again, following Theorem~\ref{theorem:submod-mutinfo-mon-submod-prop}, we know that $I_f(A; P)$ is supermodular in $A$ if the third-order partial derivatives are not positive, i.e. $f^{(3)}(i,j,k;A) = f^{(2)}(j,k;A \cup i) - f^{(2)}(j,k;A) \leq 0$. Next, we will show that NSMIMax is inapproximable. For this, let $P = V$. Then $I_f(A; P) = f(A)$ and NSMIMax becomes maximizing $g(A) - f(A)$ which is a difference of submodular functions.  We then invoke Theorem 5.6 from~\citep{iyer2015submodular} which shows that the problem of maximizing the difference of submodular functions cannot be approximated up to any polynomial factor. 
\end{proof}
NSMIMax has a number of issues though. First, for most common submodular functions, the third-order partial derivative is actually non-negative instead of non-positive and hence $I_f(A; P)$ will actually be submodular. As a result, this problem is likely to be in-approximable. Secondly, when there is no privacy constraint, $I_f(A; P) = I_f(A; \emptyset) = 0$ and hence we do not get back the problem of just maximizing the function $f$.

A more natural solution which addresses both issues is to maximize the conditional gain $f(A | P)$. Hence we consider the following problem (CGMax):
\begin{align}
\max_{A \subseteq \Omega, |A| \leq k} \lambda g(A) + f(A | P)
\end{align}
This problem has several nice properties. When there is no private set, i.e. $P = \emptyset$, $f(A | P) = f(A)$. Secondly, $f(A | P)$ is monotone, normalized, non-negative and submodular when $f$ satisfies all these properties, and hence CGMax admits a $1 - 1/e$ approximation guarantee~\citep{nemhauser1978analysis}. Combinng the above with the proof of Lemma~\ref{lemma-approx-nsmimax} concludes the proof of Lemma~\ref{lemma-approx-nsmimax-cgmax}. 

\subsubsection{Privacy Preserving Summarization: Formulation as a constrained problem}
Similar to the query based summarization, we study an alternative formulation of privacy preserving summarization. Again, the goal of this formulation is to decouple privacy as a constraint and thereby have explicit control on the amount of privacy we would like to enforce. 
\begin{align} \label{opt:scsk-privacy}
\max_{A \subseteq \Omega} g(A) \;\text{s.t}\; I_f(A; P) \leq \epsilon, |A| \leq k  
\end{align}
This problem is related to NSMIMax above, where instead of minimizing $I_f(A; P)$, we add it to the constraint. Moreover, this problem is an instance of SCSK as long as the function $I_f(A; P)$ is submodular. Recall that $I_f(A; P)$ is submodular if the third-order partial derivative of $f$ is non-negative, which is true for many submodular functions like facility location and set cover. As a result, equation~\eqref{opt:scsk-privacy} is more tractable compared to NSMIMax. Similar to the constrained formulation in the query based case, we have explicit control over the privacy via the parameter $\epsilon$ in the constraint. Also, note that when $P = \emptyset$, the constraint $I_f(A; P) \leq \epsilon$ disappears (since $I_f(A; P) = 0$) and we get cardinality constrained submodular maximization. Finally, the constrained formulation (equation~\eqref{opt:scsk-privacy}) may not have a bounded approximation factor (when $f$ does not have third-order partial derivatives negative) and in that case, CGMax is a preferred approach.

\subsubsection{Privacy Preserving and Query Based Summarization: Analysis of CSIMax and Proof of Lemma~\ref{lemma-approx-csmimax}}
The first approach for this problem is:
\begin{align}
    \max_{A \subseteq \Omega, |A| \leq k} I_f(A; Q) - I_f(A; P) + \lambda g(A)
\end{align}
When $P = \emptyset$, we get back SMIMax. However, when $Q = V$, we obtain $f(A) + \lambda g(A) - I_f(A; P)$, which is similar to NSMIMax. Furthermore, when $Q = \emptyset$, we have $\lambda g(A) - I_f(A; P)$ which is exactly NSMIMax. From the previous section, we know that NSMIMax is inapproximable, this implies that this problem will also be in-approximable in the worst case. This is true even when the third-order partial derivatives are non-negative which is true for a rich subclass of submodular functions. Furthermore, this not yield the desirable $f(A | P)$ in the privacy preserving case. 

As a result, we study one more formulation (which we call CSMIMax):
\begin{align}
    \max_{A \subseteq \Omega, |A| \leq k} I_f(A; Q | P) + \lambda g(A)
\end{align}
Let us look at some special cases. When $Q = V$, we get back CGMax: $f(A | P) + \lambda g(A)$. Similarly, when $P = \emptyset$, we obtain $I_f(A; Q) + \lambda g(A)$ which is SMIMax. Finally, when $P = \emptyset, Q = V$, we obtain $f(A) + \lambda g(A)$. Moreover, as we show in the following Lemma, this problem is $1 - 1/e$ approximable by the greedy algorithm for a rich subclass of submodular functions.

Next, we prove Lemma~\ref{lemma-approx-csmimax}
\begin{proof}
The proof of the first part is a direct corollary of Theorem~\ref{theorem:submod-mutinfo-mon-submod-prop} (the submodularity of submodular conditional mutual information). In particular, if $f^{(3)}(i,j,k;A) \geq 0$, we have that $I_f(A; Q | P)$ is submodular in $A$ for a given $Q$. Also note that, $I_f(A; Q | P)$ is always monotone, and hence if $g$ is monotone submodular, it implies that $I_f(A; Q | P) + \lambda g(A)$ is monotone submodular and hence the $1 - 1/e$ approximation holds from~\citep{nemhauser1978analysis}.
\end{proof}
Finally, we provide some intuition into maximizing $I_f(A; Q | P)$, in addition to the fact that in special cases, it yields both CGMax and SMIMax. Note that $I_f(A; Q | P) = I_f(A; Q) - I_f(A; Q; P)$. This implies that we are trying to maximize $I_f(A; Q)$ while minimizing $I_f(A; Q; P)$. Maximizing $I_f(A; Q)$ will try to obtain a set $A$ which is similar to the query set $Q$. On the other hand, minimizing $I_f(A; Q; P)$ will try to have set $A$ as independent as possible from $P$ (since $A$ needs to be close to $Q$ to maximize the first term). Another way to look at this is by noting that $I_f(A; Q | P) = f(A \cup P) - f(A \cup P \cup Q)$ plus some terms which are independent of $A$. As a result, maximizing $I_f(A; Q | P)$ is equivalent to maximizing $f(A \cup P) - f(A \cup P \cup Q)$, or in other words, maximizing $f(A \cup P)$ while minimizing $f(A \cup P \cup Q)$. Maximizing $f(A \cup P)$ means the set $A$ should be as independent of $P$ as possible, while minimizing $f(A \cup P \cup Q)$ means that $A$ should be as similar to $Q$ as possible. 

\subsubsection{Privacy Preserving and Query Based Summarization: Constrained Formulation}
Finally, similar to both privacy preserving and query based summarization, we study constrained formulations for joint privacy preserving and query based summarization. Recall that for query based summarization, we had a constraint $f(A | Q) \leq \epsilon$ while in the privacy preserving summarization, the constraint was $I_f(A; P)$. We can add both of these as constraints:
\begin{align} \label{opt:scsk-query-privacy}
\max_{A \subseteq \Omega} g(A) \;\text{s.t}\; f(A | Q) \leq \epsilon_1, I_f(A; P) \leq \epsilon_2, |A| \leq k  
\end{align}
The thresholds $\epsilon_1, \epsilon_2$ allows direct control over the query similarity and privacy. However, similar to the privacy preserving summarization, this requires $I_f(A; P)$ to be submodular, which in turn requires $f$ to have non-negative third-order partial derivatives. 

When $f$ does not satisfy this condition, we can modify the optimization problem slightly to make it tractable. In particular:
\begin{align} \label{opt:scsk-query-privacy-mod}
\max_{A \subseteq \Omega} g(A) + \lambda_2 f(A | P) \;\text{s.t}\; f(A | Q) \leq \epsilon_1,  |A| \leq k  
\end{align}
Note that since both the objective and the constraint are monotone submodular, this problem is an instance of SCSK~\citep{iyerSubmodularOptimizationSubmodular}.

\subsection{Clustering and Partitioning using the Multi-Set Submodular Mutual Information}
In this section, we build upon the discussion of section~\ref{sec:optimization-problems} for the clustering and partitioning applications using the multi-set mutual information functions. 
\subsubsection{Clustering or Minimizing Multi-Set Mutual Information}
First we show that with the total correlation, we get a known submodular partitioning problem.
\begin{align} \label{opt:min-totcorr}
\min_{A_1, \dots, A_k \; \mathrm{s.t} \;\cup_{i=1}^{k} A_i = \Omega \; , \; A_i \cap A_j = \emptyset \; \forall i,j} C_f(A_1, \dots, A_k)
\end{align}
This problem is equivalent to minimizing $\sum_{i = 1}^k f(A_i) - f(\cup_{i = 1}^k A_i)$. Since, $\cup_{i = 1}^k A_i = \Omega$, the observation below follows.
\begin{observation} 
The problem stated in equation \ref{opt:min-totcorr} is equivalent to the submodular multiway partition problem with the only difference being in the values of the respective objective functions, which always differ from each other by the same constant value and which is equal to $f(\Omega) = f(\cup_{i=1}^k A_i)$.
\end{observation}
Next, let us look at minimizing the $k$-way submodular mutual information Unfortunately, this is not very interesting and we explain this with an example. With the Set Cover function, $I_f(A_1, \dots, A_k) = w(\cap_{i = 1}^k \gamma(A_i))$, and the objective can be minimized with a trivial partition where sets $A_1, \cdots, A_{k-1}$ cover almost similar items but one set $A_k$ covers very different concept(s). This is not a very desirable clustering. 

\subsubsection{Diverse Partitioning or Maximizing Multi-Set Mutual Information}
Again, we start with the total correlation.
\begin{align} \label{opt:max-totcorr}
\max_{A_1, \dots, A_k \; \mathrm{s.t} \;\cup_{i=1}^{k} A_i = \Omega \; , \; A_i \cap A_j = \emptyset \; \forall i,j} C_f(A_1, \dots, A_k)
\end{align}
Equation~\ref{opt:min-totcorr} is equivalent to maximizing $\sum_{i = 1}^k f(A_i) - f(\cup_{i = 1}^k A_i)$. Since, $\cup_{i = 1}^k A_i = \Omega$, the observation below follows.
\begin{observation} \label{lemma:max-totcorr-submod-welfare}
The problem stated in equation \ref{opt:max-totcorr} is equivalent to the submodular welfare problem with the only difference being in the values of the respective objective functions, which always differ from each other by the same constant value equal to  $f(\Omega) = f(\cup_{i=1}^k A_i)$.
\end{observation}
Next, look at the problem of maximizing the multi-set mutual information. 
\begin{align} 
\max_{A_1, \dots, A_k \; \mathrm{s.t} \;\cup_{i=1}^{k} A_i = \Omega \; , \; A_i \cap A_j = \emptyset \; \forall i,j} I_f(A_1, \dots, A_k)
\end{align}
This problem is interesting since it is related to robust (i.e., worst case) partitioning~\citep{weiMixedRobustAverage2015}. To understand this better, we see the facility location function. Recall that 
\begin{align}
I_f(A_1; \cdots; A_k) = \sum_{i \in \Omega} \min(\max_{a_1 \in A_1} s(i, a_1), \cdots, \max_{a_k \in A_k} s(i, a_k))    
\end{align}
The expression is similar in the case of set-cover. We see that maximizing $I_f$ is  related to maximizing the minimum among the functions and hence is a robust partitioning approach. The worst case (robust) partitioning tries to solve: 
\begin{align}
   \max_{A_1, \dots, A_k \; \mathrm{s.t} \;\cup_{i=1}^{k} A_i = \Omega \; , \; A_i \cap A_j = \emptyset \; \forall i,j} \min (f(A_1), \dots , f(A_k)) 
\end{align}
This objective is similar to the multi-set mutual information for the set cover and facility location, except that we have the sum of the minimums instead of the minimum of the sums in both cases. Moreover, we can also view the multi-set mutual information as maximizing a lower bound of the robust objective~\citep{weiMixedRobustAverage2015}.

\subsection{Proof of Lemma~\ref{psudeometric-opt-lemma} (Minimization of Submodular Information Metric)}\label{sec:proof-psudeometric-opt-lemma}
\begin{proof}
We can obtain an approximation to above problem by taking a look at the bounds show in Lemma \ref{lemma:psuedometric-bound} and leveraging the additive version of the submodular hamming metric $D^{SHA}(A,B) = f(A\setminus B) + f(B \setminus A)$. Note that for a fixed $S_i$, $D^{SHA}(A, S_i)$ is a submodular function in $A$ and hence it can exactly minimized in polynomial time~\citep{fujishige2005submodular}. Moreover, since $D^{SHA}(A, S_i)$ approximates $D_f(A, S_i)$ upto a factor of $1 - \kappa_f$, we get the resulting approximation guarantee.
\end{proof}
}
\end{document}